\documentclass[11pt,a4paper]{article}

\usepackage[margin=1.2in]{geometry}

\usepackage[T1]{fontenc}
\usepackage[utf8]{inputenc}
\usepackage[english]{babel}
\usepackage{lmodern} 

\usepackage{amsmath,amssymb,amsfonts,amsthm}
\usepackage{mathtools}
\usepackage{mathrsfs}
\usepackage{bm}

\usepackage{graphicx}
\usepackage[small]{caption}
\usepackage{subcaption}
\usepackage{booktabs}
\usepackage{float}

\usepackage[ruled,vlined]{algorithm2e}

\usepackage{tikz}
\usepackage{pgfplots}
\pgfplotsset{compat=1.18}
\usepgfplotslibrary{fillbetween}
\usetikzlibrary{positioning,shapes,decorations.pathreplacing,shapes.geometric,fit}

\usepackage[numbers,sort&compress]{natbib}
\usepackage{url}
\usepackage[hidelinks]{hyperref}
\usepackage[nameinlink,capitalize]{cleveref}

\usepackage{microtype}

\newtheorem{example}{Example}
\newtheorem{theorem}{Theorem}
\newtheorem{corollary}[theorem]{Corollary}
\newtheorem{lemma}{Lemma}
\newtheorem{definition}{Definition}
\newtheorem{proposition}{Proposition}

\newcommand{\ost}{\omega^*}
\newcommand{\odg}{{\omega_\dag}}

\newcommand{\expl}[1]{\text{\scriptsize{}#1}}

\title{Epistemic Filtering and Collective Hallucination: A Jury Theorem for Confidence-Calibrated Agents}

\author{
  Jonas Karge \\
  Technische Universit\"at Dresden \\
  \texttt{jonas.karge@tu-dresden.de}
}

\date{} 

\begin{document}

\maketitle


\begin{abstract}
We investigate the collective accuracy of heterogeneous agents who learn to estimate their own reliability over time and selectively abstain from voting. While classical epistemic voting results, such as the \textit{Condorcet Jury Theorem} (CJT), assume fixed participation, real-world aggregation often benefits from allowing agents to say ``I don't know.'' We propose a probabilistic framework where agents engage in a \textit{calibration} phase, updating beliefs about their own fixed competence, before facing a final confidence gate that determines whether to vote or abstain. We derive a non-asymptotic lower bound on the group's success probability and prove that this \textit{selective participation} generalizes the asymptotic guarantees of the CJT to a sequential, confidence-gated setting. Empirically, we validate these bounds via Monte Carlo simulations. While our results are general, we discuss their potential application to AI safety, outlining how this framework can mitigate \textit{hallucinations} in collective LLM decision-making.
\end{abstract}

\section{Introduction}

A fundamental problem in \emph{Artificial Intelligence} is the aggregation of noisy information from heterogeneous sources. When this problem is treated from an \emph{epistemic social choice} perspective, the underlying aggregation mechanism is \emph{voting}, and the objective is to identify an underlying ground truth that emerges as the collective belief.

This challenge is classically addressed by the \textit{Condorcet Jury Theorem} (CJT), which provides probabilistic guarantees that a majority of fallible agents can collectively identify the truth with high probability \cite{Condorcet}. In recent work from the field of \textit{symbolic AI}, voting generalizes \textit{belief revision} to multi-agent settings. Approaches like \textit{belief fusion} \cite{konieczny1999merging} and \textit{judgment aggregation} \cite{list2002aggregating} use voting rules to merge conflicting knowledge bases while preserving logical consistency. The CJT has been formalized in this context to track truth in incomplete information settings \cite{Everaere} and to model the convergence of non-expert opinions to a true propositional state \cite{singleton2024truth}. Moreover, in \textit{Statistical AI}, the CJT underpins \textit{ensemble methods}. Early work by \citeauthor{levin1990statistical} (\citeyear{levin1990statistical}) and \citeauthor{hansen1990neural} (\citeyear{hansen1990neural}) modeled neural network ensembles as voting systems, directly influencing the development of \textit{boosting} \cite{schapire1990design} and \textit{bagging} \cite{breiman1996bagging}. These methods rely on CJT-like diversity assumptions to ensure that an aggregate of weak learners outperforms any single learner \cite{Dietterich,Lam}. Nowadays, building upon these principles in safety-critical applications, the CJT principles are used to derive voting-based ensemble scores that improve efficiency in COVID-19 identification \cite{srivastava2022ensemble} and enhance detection of colon cancer in histopathology \cite{srivastava2023cjt}.

\vspace{-0.25cm}

\paragraph{The Condorcet Jury Theorem.}
As a foundational theorem in voting theory, the classical CJT \cite{Condorcet} assumes that agents have homogeneous competencies, are more likely to vote for the correct alternative than for the incorrect option, and independent in their decision-making. Additionally, the theorem assumes that agents choose exactly one alternative from two options under majority voting. Under these conditions, leveraging the \textit{wisdom of the crowd effect}, the classical CJT establishes that the probability of majority voting identifying the correct alternative 

(1) increases monotonically with the number of agents (non-asymptotic guarantee),  and 

(2) converges to 
1 as the number of agents approaches infinity (asymptotic guarantee).




\subsection{Our Voting Framework: Epistemic Filtering.}

Analogously to the original CJT, we model the underlying decision problem as a binary vote over two alternatives.
Our framework generalizes standard epistemic voting by introducing a \textit{calibration mechanism}. There are $N$ agents who act over $T$ rounds, facing a sequence of independent tasks. Each agent $a_i$ possesses a fixed but unknown \textit{distribution-level} reliability $p_i \in [0,1]$ (their inherent probability of solving a random task correctly).

Crucially, agents do not ``learn'' to perform the task better (reinforcement learning); rather, they \textit{calibrate their confidence} to estimate their static competence $p_i$. After each round $t < T$, agents receive private feedback, update their belief about $p_i$, and compute a confidence score. They follow an explicit target: publish a vote only if confidence exceeds a threshold $\tau_{\mathrm{abstain}}$. This creates a \textit{filtering effect}: the learning phase removes low-competence agents from the final electorate. In the terminal round $T$, we aggregate only the published votes of this calibrated sub-group.

\begin{example}
Figure \ref{LearningExample} illustrates the calibration process. Four agents begin with a prior belief about their reliability. Over $T=9$ calibration rounds, they update their internal confidence $\mathcal{C}_{i,t}$ based on feedback. At the final decision round ($t=10$), only agents whose confidence exceeds $\tau_{\mathrm{abstain}}=0.5$ (Agents 2 and 3) enter the electorate. The low-performing Agent 1 (blue) correctly identifies its own unreliability and abstains, preventing a potential error.

\begin{figure}[ht]
    \centering
\begin{tikzpicture}[scale=1]
\begin{axis}[
    width=13cm,
    height=7cm,
    xlabel={Time step $t$},
    ylabel={Confidence $\mathcal{C}_{i,t}$},
    ymin=0, ymax=1,
    xmin=1, xmax=11,
    xtick={1,...,10},
    ytick={0,0.25,0.5,0.75,1},
    grid=both,
    grid style={gray!30},
    title={Calibration and Selective Voting},
    legend style={
        at={(1.01,0.5)},
        anchor=west,
        font=\small,
        cells={anchor=west},
        draw=black,
        fill=white,
        rounded corners
    },
    clip=false,
    legend image code/.code={},
]

\addplot [
    draw=none,
    fill=gray!10
] coordinates {(1,0) (1,1) (9,1) (9,0) (1,0)};
\node at (axis cs:5,1) [anchor=north, font=\small] {Calibration Phase};

\addplot[black, thick, dashed] coordinates {(1,0.5) (10,0.5)};
\addlegendentry{$\tau_{\mathrm{abstain}}=0.5$}

\addplot[blue, thick] coordinates {
    (1,0.45) (2,0.52) (3,0.58) (4,0.63) (5,0.68)
    (6,0.55) (7,0.40) (8,0.30) (9,0.25) (10,0.20)
};
\addlegendentry{\textcolor{blue}{Agent 1}}
\addplot[only marks, mark=o, mark size=3, blue] coordinates {
    (1,0.45) (7,0.40) (8,0.30) (9,0.25) (10,0.20)
};
\addplot[only marks, mark=*, mark size=3, blue] coordinates {
    (2,0.52) (3,0.58) (4,0.63) (5,0.68) (6,0.55)
};

\addplot[red, thick] coordinates {
    (1,0.35) (2,0.40) (3,0.48) (4,0.52) (5,0.56)
    (6,0.62) (7,0.70) (8,0.75) (9,0.78) (10,0.82)
};
\addlegendentry{\textcolor{red}{Agent 2}}
\addplot[only marks, mark=o, mark size=3, red] coordinates {
    (1,0.35) (2,0.40) (3,0.48)
};
\addplot[only marks, mark=*, mark size=3, red] coordinates {
    (4,0.52) (5,0.56) (6,0.62) (7,0.70) (8,0.75) (9,0.78) (10,0.82)
};

\addplot[green!70!black, thick] coordinates {
    (1,0.60) (2,0.65) (3,0.66) (4,0.67) (5,0.68)
    (6,0.70) (7,0.72) (8,0.74) (9,0.77) (10,0.79)
};
\addlegendentry{\textcolor{green!50!black}{Agent 3}}
\addplot[only marks, mark=*, mark size=3, green!70!black] coordinates {
    (1,0.60) (2,0.65) (3,0.66) (4,0.67) (5,0.68)
    (6,0.70) (7,0.72) (8,0.74) (9,0.77) (10,0.79)
};

\addplot[purple, thick] coordinates {
    (1,0.30) (2,0.35) (3,0.38) (4,0.40) (5,0.42)
    (6,0.44) (7,0.46) (8,0.47) (9,0.48) (10,0.49)
};
\addlegendentry{\textcolor{purple}{Agent 4}}
\addplot[only marks, mark=o, mark size=3, purple] coordinates {
    (1,0.30) (2,0.35) (3,0.38) (4,0.40) (5,0.42)
    (6,0.44) (7,0.46) (8,0.47) (9,0.48) (10,0.49)
};

\addlegendimage{} 
\addlegendentry{Abstain ($\circ$)}
\addlegendimage{} 
\addlegendentry{Vote ($\bullet$)}

\draw[dashed, thick, black] (axis cs:10.1,0.8) ellipse [x radius=0.6, y radius=0.15];

\end{axis}
\end{tikzpicture}
    \caption{Learning and Final Vote Visualization.}
\label{LearningExample}
\end{figure}
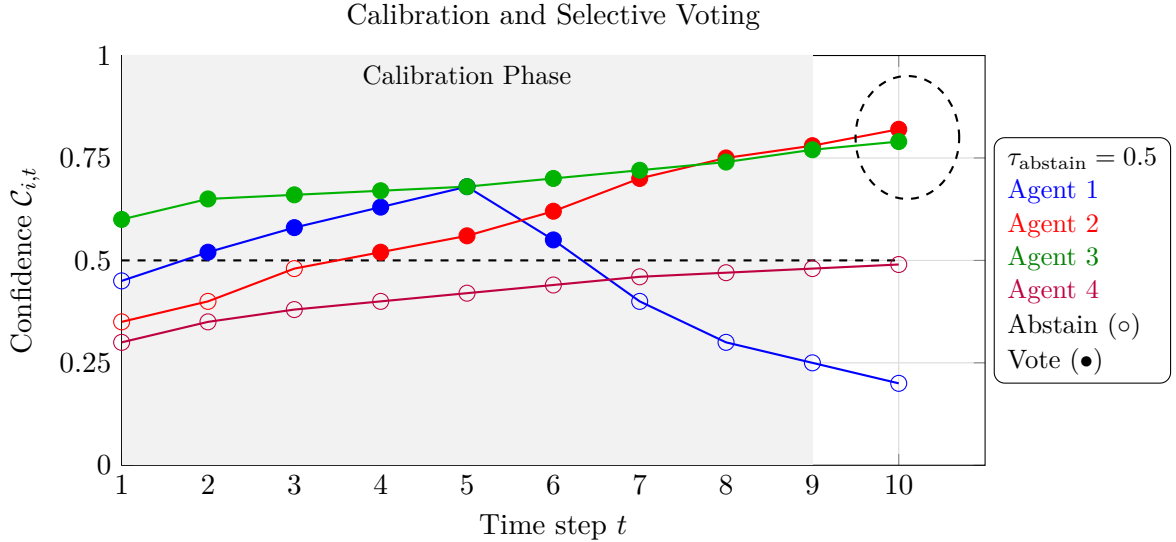
\end{example}

\subsection{Related Work}

\paragraph{Generalizations of the Condorcet Jury Theorem.}

Extensions of the CJT typically either allow for more general voting rules or weaken the \emph{homogeneity} or \emph{independence} assumption, thereby extending the \emph{asymptotic guarantee} of the CJT.
\citeauthor{List} (\citeyear{List}) generalized the result to \textit{plurality voting} over multiple alternatives, while \citeauthor{Everaere} (\citeyear{Everaere}) extended it to \textit{approval voting}, assuming that the probability of approving the correct alternative exceeds that of any incorrect one.
Crucially for our work, \citeauthor{Owen} (\citeyear{Owen}) relaxed the homogeneity assumption, proving that convergence requires only that the \textit{average} reliability of the group exceeds 0.5.
Other works have addressed independence by introducing an \textit{opinion leader}, swaying the electorate towards one of the alternatives \cite{Boland,Goodin}, or limiting pairwise correlations \cite{Ladha,Pivato}.


In this work, we likewise generalize the CJT by (1) extending the one-shot decision model to a \emph{sequential setting}, (2) permitting \emph{heterogeneous agents}, (3) whose decisions can \emph{depend on their previous choices}, and (4) allowing agents to \emph{abstain} from voting.

\vspace{-0.25cm}

\paragraph{Delegation and Liquid Democracy.}

Permitting agents to abstain can be viewed as delegating their vote to the remaining group. This connects to \textit{liquid democracy}, where agents transitively delegate votes to those deemed more capable.

However, there is a critical distinction to the framework proposed here. Standard liquid democracy models constrain delegation to \textit{local social ties} \cite{Kahng,Alouf}. \citeauthor{Kahng} (\citeyear{Kahng}) prove that such local mechanisms cannot guarantee to outperform direct voting because they risk \textit{weight concentration}, where a few proxies accumulate massive influence, destroying the independence required for the ``wisdom of the crowd.''

Our framework avoids this failure mode by using \textit{calibration-based abstention} rather than social delegation. Agents do not delegate to a specific neighbor; they effectively ``delegate'' to the collective mean by withdrawing their noise, keeping influence distributed among all confident agents.

\vspace{-0.25cm}
\paragraph{Contributions.}

We develop a probabilistic framework for collective decision-making with calibrated abstention. Our main contributions are:
(1) A sequential decision model in which agents undergo a \textit{confidence calibration} phase to estimate their own competence and \textit{selectively participate} via a confidence gate;
(2) A generalization of the CJT to this sequential setting, proving asymptotic convergence for a heterogeneous, confidence-gated electorate; and
(3) A non-asymptotic lower bound on the probability of a correct majority.
(4) Finally, we demonstrate the utility of these bounds via Monte Carlo simulations.

\subsection{Motivational Scenario}

To further motivate the framework proposed in this work, we briefly discuss its ties to the \emph{hallucination problem} for Large Language Models (LLMs), a phenomenon in which they produce confident yet factually incorrect outputs. A recent OpenAI study by \citeauthor{kalai2025} (\citeyear{kalai2025}) argues that many LLM \textit{hallucinations} are ordinary \textit{binary classification} errors: under binary grading (responses scored as simply correct vs.\ incorrect), systems are implicitly rewarded for guessing and penalized for abstaining, which can encourage confident yet wrong answers. They urge a shift in how we \textit{evaluate and deploy} LLMs by introducing explicit confidence targets and granting credit for admitting not to know the answer (``IDK'') when uncertainty is high, so models answer only when sufficiently confident.

This presents a theoretical challenge: How can we aggregate agents who are incentivized to selectively abstain? We answer this with a probabilistic voting framework: agents undergo a calibration phase to learn their own competence, \textit{abstain unless confident}, and only the resulting ``filtered'' electorate votes at the final decision point. This operationalizes confidence-targeted evaluation and, via concentration bounds, yields probabilistic guarantees on the correctness of the group decision in \Cref{maintheorem} and a bound on the probability for a group of agents hallucinating in \Cref{cor:hallucination-bound}.

In this context, our approach synthesizes three active lines of LLM research. First is \textit{ensemble-style decoding}, most notably Self-Consistency \cite{wang2022selfconsistency}, which samples multiple reasoning paths (mimicking diverse agents) and aggregates answers to improve accuracy. Second is the tradition of \textit{classification with a reject option}, ranging from Chow’s optimal abstention rule \cite{chow1957optimum} to modern deep selective prediction \cite{geifman2019selectivenet}; our framework operationalizes this by treating agents as selective classifiers with calibrated refusal. Third, recent studies on \textit{LLM uncertainty} \cite{kadavath2022know} suggest models’ self-evaluated competence correlates with correctness. In our terms, an agent's self-evaluated competence can serve as a prior in the calibration phase.

 \section{Preliminaries}

\paragraph{The Formal Voting Framework.}
As the basic building blocks of our binary voting framework, let there be two alternatives, 
$\mathcal W=\{\ost,\odg\}$, and encode $\ost$ as $+1$ and $\odg$ as $-1$.
 For a single decision instance, we denote the (a priori unknown) true world state by a random variable $\Omega \in \mathcal{W}$. However, our setting involves a sequence of $T$ tasks. Therefore, for each time step $t \in \{1, \dots, T\}$, we let $\Omega_t \in \mathcal{W}$ denote the ground truth of the task at $t$, and we write $\boldsymbol{\Omega} = (\Omega_1, \dots, \Omega_T) \in \mathcal{W}^T$ for the complete sequence of ground truths. Probabilities of internal private choices and public votes at time $t$ are taken \textit{conditional on} $\Omega_t$. When we say “under the true state,” we generally refer to the event $\{\Omega_t=\ost\}$ (or $\{\Omega_t=\odg\}$) for the current task.

Let $\mathcal{A}=\{a_1,\dots,a_N\}$ be the set of agents with $N := \vert\mathcal{A}\vert$. A single majority-voting instance (e.g., the collective decision at time $T$) is modeled as a partial function $v:\mathcal A\to\mathcal W$ defined on the set of non-abstainers. The score of $\omega\in\mathcal W$ is $\#_v\,\omega:=|\{i\in\mathcal A: v(i)=\omega\}|$, and the winner is any $\omega$ with strictly higher score than its competitor, thereby treating ties as non-wins \citep{karge2024lead}.

\paragraph{Confidence Measure and Updating.}

In our model, every agent $a_i$ faces a sequence of i.i.d.\ tasks from $\mathcal D$ over $T$ time steps. Rounds $t\in\{1,\dots,T-1\}$ form a \textit{learning phase} (agents receive feedback and update beliefs); round $t=T$ is the \textit{decision phase} (only this round’s public votes are aggregated). We posit a fixed, unknown parameter $p_i\in[0,1]$ capturing \emph{distribution-level reliability}:
\[
p_i\ :=\ \mathbb P\big(\text{$a_i$ is correct on a random task from }\mathcal D\big).
\]
Intuitively, $p_i$ represents the agent's inherent \textit{epistemic type} or domain-specific competence (e.g., a juror's aptitude for interpreting evidence in a specific class of legal cases). Under the assumption that decision tasks are drawn i.i.d., $p_i$ acts as a structural invariant; this implies that an agent's past frequency of success is a valid signal for learning their own type, rather than a reflection of changing ability.
Equivalently, for each round $t$, the correctness indicator $\mathbf 1\{X_{i,t}=\Omega_t\}$ is Bernoulli($p_i$) and independent across $t$ (and across agents), conditional on $\{p_i\}$ and the sequence of ground truths $\{\Omega_t\}$. We denote the \textit{average probability} of making the correct choice by $\bar{p}$, where $\bar{p} = \frac{1}{N}\sum_{i=1}^N p_i$.

The agent's \textit{probabilistic belief} about its true reliability $p_i$ at time $t$ is represented by a random variable $\Psi_{i,t}$, which follows a Beta distribution:
$$
\Psi_{i,t} \sim \text{Beta}(\alpha_{i,t}, \beta_{i,t})
$$
where $\alpha_{i,t} > 0$ and $\beta_{i,t} > 0$ are the current shape parameters (often interpreted as pseudo-counts of correct and incorrect outcomes, respectively). These parameters are updated dynamically as the agent receives feedback on their choices in each binary decision problem.
The agent's \textit{confidence} in its reliability, denoted by $\mathcal{C}_{i,t}$, is derived from this belief, represented by the distribution of $\Psi_{i,t}$.

Specifically, $\mathcal{C}_{i,t}$ is defined as the posterior probability that $p_i$ is greater than a critical probability threshold, ${p_{\text{critical},i}}$. That is:
\[
\mathcal{C}_{i,t} := \mathbb{P}(\Psi_{i,t} > {p_{\text{critical},i}})
\]
For example, for binary decisions, a natural choice for the critical probability threshold is ${p_{\text{critical},i}} = 0.5$. In this context, $\mathcal{C}_{i,t}$ represents the agent's belief that its true reliability $p_i$ is greater than $0.5$. If, for instance, $\mathcal{C}_{i,t} = 0.65$, it means the agent believes there is a $65\%$ probability that its true reliability $p_i$ is greater than $0.5$.

\paragraph{Fundamentals of the Beta Distribution.}

The Beta distribution is a continuous probability distribution defined on the interval $[0, 1]$, parameterized by two positive shape parameters $\alpha$ and $\beta$.

The probability density function (PDF) of the Beta distribution is given by:
\[
f(x; \alpha, \beta) = \frac{x^{\alpha-1} (1-x)^{\beta-1}}{B(\alpha, \beta)}, \quad x \in [0, 1]
\]
where $B(\alpha, \beta)$ is the \textit{Beta function}, a normalization constant defined as $B(\alpha, \beta) = \int_0^1 x^{\alpha-1} (1-x)^{\beta-1} \, dx$. The parameters $\alpha$ and $\beta$ can be interpreted as "pseudo-counts" of successes and failures, respectively, pulling the distribution toward 1 or 0 accordingly. When $\alpha = \beta = 1$, the Beta distribution reduces to the uniform distribution.

At each decision instance at time $t$, each agent $a_i$ receives feedback on whether its private choice was correct. This feedback is used to update the parameters of its Beta belief distribution. If agent $a_i$'s current belief state is $\text{Beta}(\alpha_{i,t}, \beta_{i,t})$:
\begin{itemize}
\item If agent $a_i$'s private decision at time $t$ is \textit{correct}:
\[
(\alpha_{i,t+1}, \beta_{i,t+1}) = (\alpha_{i,t} + 1, \beta_{i,t})
\]
\item If agent $a_i$'s private decision at time $t$ is \textit{incorrect}:
\[
(\alpha_{i,t+1}, \beta_{i,t+1}) = (\alpha_{i,t}, \beta_{i,t} + 1)
\]
\end{itemize}
These updates reflect the incremental learning of agent $a_i$'s true reliability $p_i$ across different time steps.


\begin{example}
Consider an agent $a_i$ whose belief about its reliability is modeled by $\Psi_{i,t} \sim \text{Beta}(2, 3)$ and later, after more decisions, by $\Psi_{i,t} \sim \text{Beta}(20, 30)$. The PDFs of these distributions are shown below.

\begin{figure}[htbp]
    \centering
    
\includegraphics[width=0.8\textwidth]{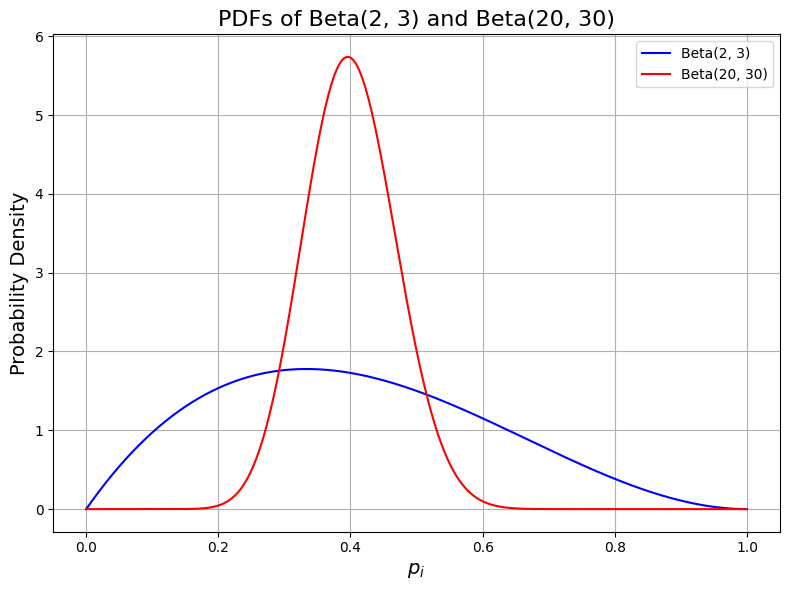}
    \caption{Beta Belief Visualization.}
\end{figure}

This illustrates how the agent's internal model assigns probability densities to different possible values of its true reliability $p_i$.
In this example, both pairs of $\alpha$ and $\beta$ parameters maintain the same proportion, but an increase in the total counts significantly reduces the agent's uncertainty. 
\end{example}

\paragraph{An Abstention Threshold.}

Building on the previous section, agent $a_i$'s confidence $\mathcal{C}_{i,t}$ at time $t$ is defined as $\mathcal{C}_{i,t} = \mathbb{P}(\Psi_{i,t}> {p_{\text{critical},i}})$, where $\Psi_{i,t} \sim \text{Beta}(\alpha_{i,t}, \beta_{i,t})$ represents the agent's current belief about its true reliability $p_i$. Agent $a_i$ makes a public choice only if its confidence exceeds a specific threshold, denoted as $\tau_{\text{abstain},i}$.
 
Recall that in the voting process, we distinguish \textit{private choices} and \textit{public votes}, and that $T$ denotes the total number of decision instances.
At each time $t$, the agent \textit{privately} decides and learns from feedback; if sufficiently confident, it also \textit{publishes} a vote.
Let $X_{i,t}\in\{-1,1\}$ be the private decision ($+1\equiv\ost$, $-1\equiv\odg$). The probability of a decision $x$ given a true state $\omega$ is defined as:

\begin{align*}
\mathbb P(X_{i,t}=x \mid \Omega_t=\omega) 
= \begin{cases} 
p_i & \text{if } x \text{ encodes } \omega, \\
1-p_i & \text{otherwise}.
\end{cases}
\end{align*}






This means that for agent $a_i$, the probability of making the correct private choice is always $p_i$. These private choices are used internally by agent $a_i$ to update its belief $\Psi_{i,t}$ (and thus $\alpha_{i,t}, \beta_{i,t}$) after each decision, regardless of whether it makes a public choice.

We assume that the private decisions of agents are made independently across trials and across agents, defined analogously to the case where we consider one-shot voting instances \citep{karge2024lead}:


\begin{definition}[Private Signal Independence]\label{independence}
A collection $(X_{i,t})_{i\in\{1,\dots,N\},\,t\in\{1,\dots,T\}}$ satisfies private signal independence if, for any sequence of ground truths $\boldsymbol{\Omega} \in \mathcal{W}^T$ and any $(v_{i,t})\in\{-1,1\}^{N\times T}$,
\begin{align*}
\mathbb P\left(\bigwedge_{i=1}^N\bigwedge_{t=1}^T (X_{i,t}=v_{i,t}) \mathrel{\Big|} \boldsymbol{\Omega}=\boldsymbol{\omega}\right) 
= \prod_{i=1}^N\prod_{t=1}^T \mathbb P\left(X_{i,t}=v_{i,t} \mid \Omega_t=\omega_t\right).
\end{align*}
\end{definition}



Agent $a_i$'s decision to make a public vote at time $t$ is captured by a \textit{binary indicator variable} $D_{\text{vote},i}(t) \in \{0, 1\}$, defined as:
\[
D_{\text{vote},i}(t) = \begin{cases}
1 & \text{if } \mathcal{C}_{i,t} > {\tau_{\text{abstain},i}} \\
0 & \text{if } \mathcal{C}_{i,t} \leq {\tau_{\text{abstain},i}}
\end{cases}
\]
Thus, $D_{\text{vote},i}(t)$ acts as a deterministic switch at each time step $t$: if agent $a_i$'s confidence $\mathcal{C}_{i,t}$ exceeds the $\tau_{\text{abstain},i}$ threshold, $D_{\text{vote},i}(t)=1$ (indicating a public vote); otherwise, $D_{\text{vote},i}(t)=0$ (indicating abstention). Note that $D_{\text{vote},i}(t)$ is determined by agent $a_i$'s internal belief state at time $t$.

Agent $a_i$'s public vote at time step $t$, denoted $V_{i,t} \in \{-1,0, 1\}$, is determined by its private choice $X_{i,t}$ only if the agent decides to vote publicly. Specifically, $V_{i,t} = X_{i,t}$ if $D_{\text{vote},i}(t)=1$, and $V_{i,t} = 0$ if $D_{\text{vote},i}(t)=0$. That is:
$$V_{i,t} = D_{\text{vote},i}(t) \cdot X_{i,t}$$
Given this definition, the overall probability for agent $a_i$'s public vote to align with the true state is:

\begin{align*}
&\mathbb P(V_{i,t}=+1 \mid \Omega_t=\ost)=\mathbb P(D_{\text{vote},i}(t)=1 \mid \Omega_t=\ost)\,p_i,
\end{align*}

and analogously under $\Omega_t=\odg$. 
Here, $\mathbb{P}(D_{\text{vote},i}(t)=1)$ represents the probability that agent $a_i$'s confidence $\mathcal{C}_{i,t}$, based on its accumulated experience up to time $t$, is sufficiently high for it to make a public choice at this time step. Note that in our analysis later, only $V_{i,T}$ enters the final tally; earlier $V_{i,t}$ ($t<T$) are specific to the learning period and not aggregated.

The abstention condition, $\mathcal{C}_{i,t} \leq \tau_{\text{abstain},i}$, can be expressed explicitly using the properties of the Beta distribution and the regularized incomplete Beta function such that $\mathbb{P}(\Psi_{i,t} > {p_{\text{critical},i}}) \leq {\tau_{\text{abstain},i}}$ is equivalent to:
$$1 - I_{{p_{\text{critical},i}}}(\alpha_{i,t}, \beta_{i,t}) \leq {\tau_{\text{abstain},i}}$$
where $I_x(\alpha, \beta)$ is the regularized incomplete Beta function. 

\paragraph{The Regularized Incomplete Beta Function.}

The regularized incomplete Beta function \( I_x(\alpha, \beta) \) is defined as:
\[
I_x(\alpha, \beta) = \frac{1}{B(\alpha, \beta)} \int_0^x t^{\alpha - 1}(1 - t)^{\beta - 1} dt
\]
where \( B(\alpha, \beta) \) is the complete Beta function:
\[
B(\alpha, \beta) = \int_0^1 t^{\alpha - 1}(1 - t)^{\beta - 1} dt
\]
The regularized function \( I_x(\alpha, \beta) \) represents the cumulative distribution function (CDF) of the Beta distribution, such that:
\[
\mathbb{P}(\Psi \leq x) = I_x(\alpha, \beta), \quad \text{and} \quad \mathbb{P}(\Psi > x) = 1 - I_x(\alpha, \beta).
\]

\begin{example}
Consider an agent with $\Psi_{i,t}\sim\mathrm{Beta}(4,3)$ and $p_{\mathrm{critical},i}=0.5$. Its confidence is
\[
\mathcal{C}_{i,t}=\mathbb{P}(\Psi_{i,t} >0.5)=1-I_{0.5}(4,3)\approx 0.65.
\]
Thus the agent assigns a $65\%$ probability that its true reliability exceeds $0.5$. If the abstention target is $\tau_{\mathrm{abstain},i}=0.5$ it would publish; if $\tau_{\mathrm{abstain},i}=0.75$ it would abstain.
\end{example}

\paragraph{Framework Summary.} 
The process unfolds over $T$ discrete steps involving a sequence of independent tasks. Rounds $t\in\{1,\dots,T-1\}$ constitute the \textit{calibration phase}, where agents learn about their own capabilities; round $T$ is the \textit{decision phase}, where public votes on the final task $\Omega_T$ are aggregated. 

\textit{Model Assumption (Static Competence vs.\ Dynamic Belief):} We assume each agent's true reliability $p_i$ is fixed (representing inherent capability). The ``learning'' process acts as a \textit{confidence calibration phase}: agents do not improve their task performance $p_i$ over time, but rather refine their \textit{estimate} of $p_i$ to distinguish whether they are sufficiently reliable to vote.

Within each round $t\le T$, the protocol proceeds as follows:
(i) \textit{Belief State:} The pre-decision belief $(\alpha_{i,t},\beta_{i,t})$ summarizes feedback from tasks $1$ to $t-1$, inducing a confidence score $\mathcal C_{i,t}$;
(ii) \textit{Gating:} The agent chooses to \textit{publish or abstain} based on its internal estimate: $D_{\mathrm{vote},i}(t)=\mathbf{1}\{\mathcal C_{i,t}>\tau_{\mathrm{abstain},i}\}$;
(iii) \textit{Task Execution:} A new task $\Omega_t$ is drawn. The private signal $X_{i,t}\in\{-1,1\}$ is realized according to the agent's fixed competence $p_i$ (Bernoulli trial), independent of the current belief state;
(iv) \textit{Feedback:} The public vote is $V_{i,t}=D_{\mathrm{vote},i}(t)\cdot X_{i,t}$. Correctness feedback on $X_{i,t}$ is observed (regardless of abstention), updating the belief to $(\alpha_{i,t+1},\beta_{i,t+1})$. 
Only the published votes $V_{i,T}$ at the final step are aggregated for the collective decision.


\section{Probabilistic Modeling of Dynamic Beliefs and Collective Voting}

This section lays the theoretical groundwork for an analysis of the final majority vote outcome, specifically bounding the probability of the correct alternative winning. This analysis integrates the dynamic evolution of agent beliefs and their abstention probabilities. Our approach proceeds in three key steps: We first introduce a probabilistic model that captures the cumulative information growth of agents across trials, formalizing this with \textit{martingales} and \textit{filtrations}. Subsequently, we derive the expected value of the aggregated votes favoring the correct alternative at the final time step. Finally, to facilitate the application of the \textit{Azuma-Hoeffding inequality}, a concentration inequality designed to provide tight bounds on the probability that a sum of dependent random variables deviates significantly from its mean, we construct a martingale difference sequence based on \textit{Doob martingales}. This sequence effectively isolates and quantifies the total deviation of the final vote, which is precisely what the Azuma-Hoeffding inequality operates on.

\paragraph{Martingales and Filtration.} To rigorously analyze the behavior of our agents and derive results akin to a jury theorem for their collective decisions, we establish a robust probabilistic framework. This framework, based on the concepts of \textit{martingales} and \textit{filtrations}, allows us to model the dynamic evolution of agents' beliefs and the accumulation of information within the system. 
We mostly follow the definitions given in \citep{durrett2019probability}, adapted to our framework.
\textit{Martingales} and \textit{filtrations} have clear intuitive meanings in thinking about betting games. For example, \textit{martingales} can be thought of as the fortune of a player betting on a fair game \citep{durrett2019probability}, formally described as a special kind of sequence of random variables whose values evolve over time. A \textit{filtration} can be thought of as sets of information increasing as the martingales evolve, capturing the game's history. Roughly speaking, in our setting, we can see the sequence of decisions made by the agents, as well as the feedback they receive in order to update their beliefs, as such a game.

In order to formally introduce both concepts we need to define \textit{probability space} and \textit{$\sigma$-algebra} first:

\begin{definition}
A \textit{probability space} is a triple $(\mathsf S, \mathcal{F}, \mathbb{P})$ where $\mathsf S$ is a set of outcomes, $\mathcal{F}$ is a collection of events (subsets of $\mathsf S$), and $\mathbb{P} : \mathcal{F} \rightarrow [0,1]$ is a probability measure. We assume that $\mathcal{F}$ is a $\sigma$-algebra, meaning it is a nonempty collection of subsets of $\mathsf S$ that satisfy \citep{durrett2019probability}:
\begin{itemize}
\item[(i)] If $A \in \mathcal{F}$, then its complement $A^c \in \mathcal{F}$.
\item[(ii)] If $A_j \in \mathcal{F}$ is a countable sequence of sets, then $\bigcup_{j=1}^\infty A_j \in \mathcal{F}$.
\end{itemize}
\end{definition}
These properties ensure that $\mathcal{F}$ comprises all subsets of $\mathsf S$ for which we can meaningfully assign probabilities. We first introduce \textit{filtration} and \textit{martingales} with respect to a single agent, extending these concepts to the multi-agent case shortly after.

\begin{definition}[Filtration]
A \textit{filtration} is a sequence of increasing $\sigma$-algebras $(\mathcal{F}_t)_{t=0, 1, \dots, T}$ on a probability space $(\mathsf S, \mathcal{F}, \mathbb{P})$.
\end{definition}

In particular, a filtration has the property $\mathcal{F}_s \subseteq \mathcal{F}_t$ for $s < t$ which signifies that information (or events that are observable) only grows over time. In our context, considering a single agent $a_i$, $\mathcal{F}_{i,0}$ corresponds to agent $a_i$'s prior belief about its reliability derived from $(\alpha_{i,0}, \beta_{i,0})$, whereas $\mathcal{F}_{i,t}$ captures all private outcomes observed by agent $a_i$ up to trial $t$, which are used to update agent $a_i$'s belief $\Psi_{i,t}$. From this, we can define an \textit{adapted stochastic process}:

\begin{definition}[Adapted Stochastic Process]
A stochastic process $(M_t)_{t=0, 1, \dots, T}$, i.e. a sequence of random variables, is said to be \textit{adapted} to a filtration $(\mathcal{F}_t)_{t=0, 1, \dots, T}$ if, for each $t$, $M_t$ is measurable with respect to $\mathcal{F}_t$.
\end{definition}

This means that the value of $M_t$ is known given the information available at time $t$ \citep{durrett2019probability}. Having defined filtrations as well as stochastic processes being adapted to filtrations, we can introduce \textit{martingales}:

\begin{definition}[Martingale]
A stochastic process, denoted 

$(M_t)_{t=0, 1, \dots, T}$, is called a \textit{martingale} with respect to a filtration $(\mathcal{F}_t)$ if it satisfies the following three conditions:
\begin{enumerate}
\item $\mathbb{E}[|M_t|] < \infty$ for all $t$. (Integrability: the expected absolute value of the process at any time is finite.)
\item $M_t$ is adapted to $(\mathcal{F}_t)$ for all $t$. (Adaptedness: the value of $M_t$ is known given the information up to time $t$.)
\item $\mathbb{E}[M_{t+1} \mid \mathcal{F}_t] = M_t$. (Fair Game Property: the expected future value of the process, given all past and present information, is equal to its current value.)
\end{enumerate}
\end{definition}

Intuitively, a martingale is your running best estimate of a \textit{target random variable} as information accumulates: conditioned on what you currently know, your expected next estimate equals the estimate you hold now. This is exactly the idea we use when we form the Doob martingale to apply the Azuma--Hoeffding inequality.

\paragraph{Multiple Agents.}

We extend our model to $N$ agents, each facing the binary decision problem $T$ times. Each agent $a_i$ (for $i=1, \dots, N$) has its own true reliability $p_i$, initial Beta parameters $(\alpha_{i,0}, \beta_{i,0})$, critical probability ${p_{\text{critical},i}}$, and abstention threshold ${\tau_{\text{abstain},i}}$.
When dealing with multiple agents, the definition of the filtration needs to encompass the information available to all agents. Assuming agents do not observe each other's private outcomes or confidence states, the natural filtration for the collective system would be the \textit{joint filtration} generated by the private outcomes of all agents representing the combined history of all private outcomes observed by all agents.






This can be captured by an event-based filtration, i.e. the sequential revelation of each individual private signal across all agents and time steps.
We define $K = NT$ as the total number of private outcomes revealed over all agents and all time steps. That is, we order the $N \times T$ individual agent-time-step private outcomes sequentially using a combined index $k = (t-1)N + i$, running from $1$ to $NT$. Let these outcomes be ordered $X_{(1)}, X_{(2)}, \dots, X_{(K)}$, where the ordering progresses through time steps first, then through agents (e.g., $X_{(1)}=X_{1,1}, \dots, X_{(N)}=X_{N,1}, X_{(N+1)}=X_{1,2}, \dots, X_{(NT)}=X_{N,T}$).

\begin{definition}
The \textit{event-based filtration} $(\mathcal{H}_k)_{k=0, 1, \dots, NT}$ is a sequence of increasing $\sigma$-algebras, where $\mathcal{H}_k$ represents the collective information generated by the first $k$ individual private outcomes revealed across the system.
Formally, for each $k \in \{1, \dots, NT\}$, the $\sigma$-algebra $\mathcal{H}_k$ is defined as:
\[
\mathcal{H}_k = \sigma(\{X_{(j)} : j \in \{1, \dots, k\}\})
\]
where $\mathcal{H}_0$ is the $\sigma$-algebra representing initial knowledge.
\end{definition}
It is this event-based filtration $(\mathcal{H}_k)$ that serves as the basis for constructing the Doob martingale in the context of concentration inequalities on the aggregate vote, as it accounts for the impact of each single new piece of information. Observe that the filtration $(\mathcal H_k)$ is an analyst filtration; individual agents observe only their own histories ($
\mathcal F_{i,t}\ :=\ \sigma\!\big(\{X_{i,s}:1\le s\le t\}\big),
\textit{for \:} t=0,1,\dots,T
$), but not those of other agents.

Now that we have a probabilistic model for the internal belief updating process of multiple agents based on their private decisions, we proceed to model the public votes in a next step.
Recall that we defined the random variable representing agent $i$'s net public vote at time step $t$ as $V_{i,t}$ as follows: $1$, if agent $i$ publicly votes for $\ost$; $-1$, if agent $i$ publicly votes for $\odg$; $0$, if agent $i$ abstains.


Following the notation of \citeauthor{karge2024lead} (\citeyear{karge2024lead}), the collective net public vote at the final time step $T$ is denoted by ${V}^{\ost-\odg}_{T}$ s.t. $\ost$ wins the collective vote iff ${V}^{\ost-\odg}_{T}>0$:
\begin{align*}
{V}^{\ost-\odg}_{T} & = \sum_{i=1}^N V_{i,T}.
\end{align*}

\paragraph{Expected Collective Net Public Vote at Final Time $T$.}

Throughout this subsection, we analyze the vote on the final task $\Omega_T$. We condition on the target true state being $\ost$: $\mathbb P_\ost(\cdot):=\mathbb P(\cdot\mid \Omega_T=\ost)$ and $\mathbb E_\ost[\cdot]$ accordingly. 

To characterize the collective vote at time $T$, we first determine its expectation. By linearity of expectation, the \textit{expected collective net public vote at time $T$} is the sum of the expected individual net public votes at time $T$:

$$
\mathbb{E}_\ost[V^{\ost-\odg}_{T}] = \mathbb{E}_\ost\left[\sum_{i=1}^{N} V_{i,T} \right]
= \sum_{i=1}^{N} \mathbb{E}_\ost[V_{i,T}] .
$$

For each agent $a_i$, the expected individual net public vote at time $T$, $\mathbb{E}[V_{i,T}]$, is derived as:
\[
\mathbb{E}_\ost[V_{i,T}] = \mathbb{E}_\ost[D_{\text{vote},i}(T) \cdot X_{i,T}].
\]
Since $D_{\text{vote},i}(T)$ depends on agent $a_i$'s updated belief state at time $T$ (determined by their private history up to $T-1$, thus $D_{\text{vote},i}(T)$ is $\mathcal{F}_{i,T-1}$-measurable), and by conditional independence given that $X_{i,T}$ is agent $a_i$'s private choice at time $T$, we can apply the law of total expectation:
\begin{align*}
\mathbb{E}_\ost[V_{i,T}] 
&= \mathbb{E}_\ost\!\big[\,\mathbb{E}_\ost[D_{\text{vote},i}(T) X_{i,T}\mid \mathcal{F}_{i,T-1}]\,\big]
\\&= \mathbb{E}_\ost\!\big[D_{\text{vote},i}(T)\,\mathbb{E}_\ost[X_{i,T}\mid \mathcal{F}_{i,T-1}] \big]
\\& = \mathbb{E}_\ost[D_{\text{vote},i}(T) \cdot (2p_i-1)].
\end{align*}


As $p_i$ is a constant (true reliability), the term $(2p_i-1)$ is also a constant. Therefore, we can factor it out of the expectation:
$$
\mathbb{E}_\ost[V_{i,T}] = (2p_i-1) \mathbb{E}_\ost[D_{\text{vote},i}(T)].
$$

The expected value of $D_{\text{vote},i}(T)$ is the probability that agent $a_i$ decides to vote publicly at time $T$:
\begin{align*}
\mathbb{E}_\ost[D_{\text{vote},i}(T)] = \mathbb{P}_\ost(\mathcal{C}_{i,T} > {\tau_{\text{abstain},i}})
\end{align*}
Substituting this back into the sum for $\mathbb{E}_\ost[{V}^{\ost-\odg}_{T}]$:
\begin{align*}
\mathbb{E}_\ost[{V}^{\ost-\odg}_{T}] &= \sum_{i=1}^{N} (2p_i-1) \mathbb{P}_\ost(\mathcal{C}_{i,T} > {\tau_{\text{abstain},i}}).
\end{align*}
To compute $\mathbb{P}_\ost(\mathcal{C}_{i,T} > {\tau_{\text{abstain},i}})$ for each agent $a_i$, we must sum over all possible numbers of correct outcomes ($k$) that could have occurred in the $T-1$ preceding trials \textit{for that specific agent $a_i$}. If $k$ outcomes were correct and $(T-1-k)$ were incorrect, then agent $a_i$'s Beta parameters at time $T$ would be $\alpha_{i,T} = \alpha_{i,0} + k$ and $\beta_{i,T} = \beta_{i,0} + (T-1-k)$, where $\alpha_{i,0}$ and $\beta_{i,0}$ are agent $a_i$'s initial prior parameters for its own competence $p_i$. The probability of observing $k$ correct outcomes in $T-1$ trials, given agent $a_i$'s true reliability $p_i$, follows a binomial distribution:


\begingroup
\small
\begin{align}\label{expectation}
 \mathbb{P}_\ost(\mathcal{C}_{i,T} > {\tau_{\text{abstain},i}}) = \sum_{k=0}^{T-1} \mathbf{1}\Big\{1 - I_{{p_{\text{critical},i}}}\big(\alpha_{i,0}+k,  \beta_{i,0} + \nonumber  \\ 
 (T-1-k)\big) > {\tau_{\text{abstain},i}}\Big\}
\cdot \binom{T-1}{k} p_i^k (1-p_i)^{T-1-k}.
\end{align}
\endgroup

Here, 
$\mathbf{1}$ is an indicator function, which equals 1 if the condition is true and 0 otherwise, i.e. we select only those histories (i.e., values of $k$) for which the agent's confidence exceeds the abstention threshold, contributing to the total probability of voting.

\paragraph{Martingale Difference Sequence.}

To bound the probability that the correct alternative wins the majority vote across all agents at the final time step $T$, we analyze the deviation of the total aggregated net public vote from its expected value.
For this purpose, we construct a \textit{Doob martingale}. This construction directly yields a sequence of random variables whose sum represents the total deviation, and whose individual increments (\textit{martingale differences}) have a zero conditional expectation. To formally define this martingale, which is conditioned on the accumulating information in the system, we first need to establish a comprehensive filtration for the entire collective. 
Recall that $\mathcal{H}_k$ is the event-based filtration capturing information from the first $k$ observed private outcomes, which orders all individual private outcomes observed by agents across all trials and which will serve as the underlying filtration.
We define a \textit{Doob martingale} $(M_k)_{k=0, \dots, NT}$ based on $V^{\ost-\odg}_{T}$ and the filtration $\mathcal{H}_k$:

\begin{definition}[Doob Martingale \citep{Doob2001}]
$$M_k = \mathbb{E}_\ost[V^{\ost-\odg}_{T} \mid \mathcal{H}_k]$$
\end{definition}
This martingale has the following properties:
\begin{itemize}
\item At the initial step ($k=0$), $M_0 = \mathbb{E}_\ost[V^{\ost-\odg}_{T} \mid \mathcal{H}_0] = \mathbb{E}_\ost[V^{\ost-\odg}_{T}]$ (since $\mathcal{H}_0$ contains no information about the random outcomes).
\item At the final step ($k=NT$), $M_{NT} = \mathbb{E}_\ost[V^{\ost-\odg}_{T} \mid \mathcal{H}_{NT}] = V^{\ost-\odg}_{T}$ (since all private outcomes up to time $T$ for all agents are observed, $V^{\ost-\odg}_{T}$ is fully determined).
\end{itemize}

To analyze the deviation from the mean, we consider a related martingale, i.e. the centered variant of $M_k$. Let $Z = V^{\ost-\odg}_{T} - \mathbb{E}_\ost[V^{\ost-\odg}_{T}]$ be the centered random variable representing the total deviation of interest. This centering is crucial because the Azuma-Hoeffding inequality operates on a sequence of martingale differences that have a zero conditional mean, a property that follows naturally from the centered martingale: 
$$
M'_k = \mathbb{E}_\ost[Z \mid \mathcal{H}_k] = \mathbb{E}_\ost[V^{\ost-\odg}_{T} - \mathbb{E}_\ost[V^{\ost-\odg}_{T}] \mid \mathcal{H}_k].
$$
This new martingale $M'_k$ directly represents our best estimate of the total deviation, given all the private outcomes revealed up to step $k$.
Notice that:
\begin{itemize}
    \item At the initial step ($k=0$), $M'_0 = \mathbb{E}_\ost[V^{\ost-\odg}_{T} - \mathbb{E}_\ost[V^{\ost-\odg}_{T}] \mid \mathcal{H}_0] = \mathbb{E}_\ost[V^{\ost-\odg}_{T}] - \mathbb{E}_\ost[V^{\ost-\odg}_{T}] = 0$.
    \item At the final step ($k=NT$), $M'_{NT} = \mathbb{E}_\ost[V^{\ost-\odg}_{T} - \mathbb{E}_\ost[V^{\ost-\odg}_{T}] \mid \mathcal{H}_{NT}] = V^{\ost-\odg}_{T} - \mathbb{E}_\ost[V^{\ost-\odg}_{T}]$.
\end{itemize}

The total deviation of interest, $V^{\ost-\odg}_{T} - \mathbb{E}_\ost[V^{\ost-\odg}_{T}]$, can then be expressed as the sum of the \textit{martingale differences} of $M'_k$. This is a direct application of the telescoping sum property, where the sum of successive differences of a sequence simplifies to the difference between its final and initial terms:
$$V^{\ost-\odg}_{T} - \mathbb{E}_\ost[V^{\ost-\odg}_{T}] = M'_{NT} - M'_0 = \sum_{k=1}^{NT} (M'_k - M'_{k-1}).$$

In the next section, we will apply a concentration inequality known as the \textit{Azuma-Hoeffding inequality} to derive a lower bound on $\ost$ winning the majority vote. This requires the sequence $(D_k)_{k=1, \dots, NT}$, where $D_k = M'_k - M'_{k-1}$, to form a \textit{martingale difference sequence}. This means it must satisfy the following properties \citep{ma2018complete}:

\begin{enumerate}
    \item \textit{Integrability:} $\mathbb{E}_\ost[|D_k|] < \infty$.
    \item \textit{Adaptedness:} $D_k$ is measurable with respect to $\mathcal{H}_k$.
    \item \textit{Zero Conditional Mean:} $\mathbb{E}_\ost[D_k \mid \mathcal{H}_{k-1}] = 0$.
\end{enumerate}

Observe that the zero conditional mean property is directly linked to the fair game property: if the expected future value of a martingale (given all past and present information) equals its current value, then the expected change in that martingale must necessarily be zero.
To prepare the Azuma-Hoeffding derivation, we show that:

\begin{proposition}
$(D_k)_{k=1, \dots, NT}$ forms a martingale difference sequence with respect to the filtration $\mathcal{H}_k$.
\end{proposition}

\textbf{Proof Sketch.} 
Let $Z:=V_T^{\ost-\odg}-\mathbb E_\ost[V_T^{\ost-\odg}]$ and define the Doob martingale
$M'_k:=\mathbb E_\ost[Z\mid\mathcal H_k]$ with differences $D_k:=M'_k-M'_{k-1}$.
(i) \emph{Integrability:} $V_T^{\ost-\odg}\in[-N,N]$, hence $Z$ and all $M'_k$ are bounded, so
$\mathbb E_\ost[|D_k|]<\infty$.
(ii) \emph{Adaptedness:} $M'_k$ is $\mathcal H_k$–measurable by definition of conditional
expectation; $M'_{k-1}$ is $\mathcal H_{k-1}$–measurable and therefore $\mathcal H_k$–measurable
since $(\mathcal H_k)$ is a filtration; thus $D_k$ is $\mathcal H_k$–measurable.
(iii) \emph{Zero conditional mean:} By the tower property,
$\mathbb E_\ost[D_k\mid\mathcal H_{k-1}]
=\mathbb E_\ost[M'_k\mid\mathcal H_{k-1}]-M'_{k-1}
=\mathbb E_\ost[Z\mid\mathcal H_{k-1}]-M'_{k-1}=0$.
Therefore $(D_k)_{k=1}^{NT}$ is a martingale difference sequence w.r.t.\ $(\mathcal H_k)$.
The full proof can be found in subsection \ref{differencesequence} of the appendix.


\section{Derivation and Simulation of Final Vote Accuracy Guarantees}

To find good tail estimates for the probability of $\ost$ winning the majority vote at the final time step, we utilize the one-sided variant of the Azuma-Hoeffding inequality. More specifically, the objective is to derive a lower bound for the probability $\mathbb{P}_\ost(V^{\ost-\odg}_{T} > 0)$. 

\begin{lemma}[Azuma-Hoeffding inequality \citep{azuma1967weighted}]\label{Azuma}
For a sequence of random variables $(Z_k)_{k=1, \dots, K}$ that forms a martingale difference sequence with respect to a filtration $(\mathcal{H}_k)$, if $|Z_k| \leq c_k$ for all $k$, then for any $\epsilon > 0$:
$$
\mathbb{P}\left(\sum_{k=1}^{K} Z_k \leq -\epsilon\right) \leq e^{-\frac{\epsilon^2}{2 \sum_{k=1}^{K} c_k^2}}
$$
\end{lemma}

 The sequence $(D_k)_{k=1, \dots, NT}$ forms such a martingale difference sequence with respect to the filtration $(\mathcal{H}_k)$, which we established in Proposition 1.
A central concern when applying the inequality is to show that the martingale difference sequence is, in fact, bounded. This we show in the following lemma:

\begin{lemma}\label{lemmadifferences}
Let $D_k = M'_k - M'_{k-1}$ be the martingale differences with respect to the event-based filtration $(\mathcal{H}_k)_{k=0, \dots, NT}$. For each $k \in \{1, \dots, NT\}$, there exists a bound $c_k$ such that $|D_k| \le c_k$. Specifically, if $X_{(k)}$ corresponds to agent $i$ at time $t^*$:
\begin{enumerate}
\item If $t^* < T$, then $c_k = |2p_{i}-1|$.
\item If $t^* = T$, then $c_k = 2$.
\end{enumerate}
Consequently, the sum of the squares of these bounds across all agents and time steps is:
$$
\sum_{k=1}^{NT} c_k^2 = \sum_{i=1}^{N} \left( (T-1)(2p_i-1)^2 + 4 \right).
$$
\end{lemma}

\textbf{Proof Sketch.} 
Let $M'_k:=\mathbb{E}_\ost[Z\mid\mathcal{H}_k]$ for $Z:=V_T^{\ost-\odg}-\mathbb{E}_\ost[V_T^{\ost-\odg}]$ and
$D_k:=M'_k-M'_{k-1}=\mathbb{E}_\ost[V_{i^*,T}\mid\mathcal{H}_k]-\mathbb{E}_\ost[V_{i^*,T}\mid\mathcal{H}_{k-1}]$
when the revealed outcome is $X_{(k)}=X_{i^*,t^*}$.  
If $t^*<T$, write $V_{i^*,T}=D_{\mathrm{vote},i^*}(T)\,X_{i^*,T}$ with
$X_{i^*,T}\in\{-1,+1\}$ and $D_{\mathrm{vote},i^*}(T)\in\{0,1\}$; by conditional
independence, $\mathbb{E}_\ost[X_{i^*,T}\mid\mathcal{H}_\cdot]=2p_{i^*}-1$, hence
$
D_k=(2p_{i^*}-1)\big(\mathbb{E}_\ost[D_{\mathrm{vote},i^*}(T)\mid\mathcal{H}_k]
-\mathbb{E}_\ost[D_{\mathrm{vote},i^*}(T)\mid\mathcal{H}_{k-1}]\big),
$
a difference of two probabilities, so $|D_k|\le |2p_{i^*}-1|$.
If $t^*=T$, then $\mathbb{E}_\ost[V_{i^*,T}\mid\mathcal{H}_k]=V_{i^*,T}$ and
$\mathbb{E}_\ost[V_{i^*,T}\mid\mathcal{H}_{k-1}]=D_{\mathrm{vote},i^*}(T)(2p_{i^*}-1)$, giving
$D_k=D_{\mathrm{vote},i^*}(T)\,(X_{i^*,T}-(2p_{i^*}-1))$ with $|D_k|\le 2$.
Summing squares over all events yields
$\sum_{k=1}^{NT}c_k^2=\sum_{i=1}^N\big((T-1)(2p_i-1)^2+4\big)$.
The full proof can be found in subsection \ref{bounded} of the appendix.


Intuitively, each increment 
is the one-step change in our running forecast 
when a single private outcome is revealed. 
If the revealed outcome belongs to an earlier learning round ($t^*<T$), it cannot change the final private vote $X_{i^*,T}$; it only shifts the \textit{publish probability} for agent $i^*$ through $D_{\mathrm{vote},i^*}(T)$. 
If the revealed outcome is the final one ($t^*=T$), it collapses the remaining uncertainty about $V_{i^*,T}\in\{-1,0,1\}$, moving a conditional mean by at most $2$. From this, we can derive our main theorem:


\begin{theorem}\label{maintheorem}
Consider a majority voting setting for $N$ agents over a learning period over $T-1$ rounds, and a final voting round $T$ for agents with individual competence values $p_i$,     prior pseudo–counts $(\alpha_{i,0},\beta_{i,0})$ for a Beta belief about its own reliability $p_i$,
 a critical competence level $p_{\mathrm{critical},i}\in(0,1)$ it aims to exceed, and an abstention threshold $\tau_{\mathrm{abstain},i}\in(0,1)$: the minimum posterior probability of exceeding $p_{\mathrm{critical},i}$ required to vote publicly. We assume private signals are conditionally independent across agents and time given the sequence of ground truths \(\boldsymbol{\Omega}\) (Def. \ref{independence}).
Then it is guaranteed that the success probability of the 
majority vote is at least  
\begin{align}\label{finalbound}
1 - e^{\left(-\frac{\left(\sum_{i=1}^{N} (2p_i-1) \mathbb{E}_\ost[D_{\text{vote},i}(T)]\right)^2}{2 \sum_{i=1}^{N} \left( (T-1)(2p_i-1)^2 + 4 \right)}\right)}.
\end{align}
\end{theorem}
\textbf{Proof.} 
From Lemma \ref{Azuma} and Lemma \ref{lemmadifferences}, we obtain: 
\begin{align}
\mathbb{P}_\ost&({V}^{\ost-\odg}_{T} > 0) \nonumber \\
& = 1 - \mathbb{P}_\ost({V}^{\ost-\odg}_{T} \le 0) \nonumber \\
& = 1 - \mathbb{P}_\ost\left({V}^{\ost-\odg}_{T} - \mathbb{E}_\ost[{V}^{\ost-\odg}_{T}] \le -\mathbb{E}_\ost[{V}^{\ost-\odg}_{T}]\right) \nonumber \\
& \expl{We apply the Azuma-Hoeffding inequality with $\epsilon = \mathbb{E}_\ost[{V}^{\ost-\odg}_{T}]$, } \nonumber \\ & \expl{$Z_k=D_k$, and $\sum c_k^2 = \sum_{i=1}^{N} \left( (T-1)(2p_i-1)^2 + 4 \right)$.} \nonumber \\
& \geq 1 - e^{\left(-\frac{(\mathbb{E}_\ost[{V}^{\ost-\odg}_{T}])^2}{2 \sum_{i=1}^{N} \left( (T-1)(2p_i-1)^2 + 4 \right)}\right)} \nonumber \\
& \expl{Substitute the expected total net public vote (derived previously):} \nonumber \\ & \quad \expl{$\mathbb{E}_\ost[{V}^{\ost-\odg}_{T}] = \sum_{i=1}^{N} \mathbb{E}_\ost[V_{i,T}] = \sum_{i=1}^{N} (2p_i-1) \mathbb{E}_\ost[D_{\text{vote},i}(T)].$} \nonumber \\
& = 1 - e^{\left(-\frac{\left(\sum_{i=1}^{N} (2p_i-1) \mathbb{E}_\ost[D_{\text{vote},i}(T)]\right)^2}{2 \sum_{i=1}^{N} \left( (T-1)(2p_i-1)^2 + 4 \right)}\right)}. 
\end{align}

This derived bound quantifies the minimum probability of winning the majority vote at the final time step. We briefly recall that $\mathbb{E}_\ost[D_{\text{vote},i}(T)]$, i.e. the probability of an agent to vote publicly,  can be computed from the model parameters alone, as outlined in Equation \ref{expectation}, by summing over all possible learning histories, each weighted by its binomial probability. 
Recall that we see \emph{hallucination} as a false positive: accepting an \emph{invalid} item \citep{kalai2025}. Applying the same Azuma–Hoeffding derivation under $\Omega_T=\odg$ (see subsection \ref{halluproof} in the appendix for the full proof), we obtain an upper bound on collective hallucination:

\begin{corollary}[Collective Hallucination Bound]\label{cor:hallucination-bound}
\begingroup\small
\[
\mathbb P_{\odg}\!\big(V_T^{\ost-\odg}>0\big)\ \le\
e^{\!\left(
-\frac{\Big(\sum_{i=1}^{N}(2p_i-1)\,\mathbb{E}_{\odg}[D_{\mathrm{vote},i}(T)]\Big)^2}
{2\sum_{i=1}^{N}\big((T-1)(2p_i-1)^2+4\big)}
\right)}.
\]
\endgroup
\end{corollary}

Equation \ref{finalbound} not only allows us to assess the success probability of multiple agents in practice, but also lets us derive the classical Condorcet conclusion: the probability that majority voting selects the correct alternative converges to \(1\) as the number of agents grows. Our convergence result relies on two assumptions. First, the average competence must stay bounded away from $0.5$ (i.e., $\bar{p} \ge 0.5 + \Delta p$). Second, we assume that the gate does not become arbitrarily strict for new agents. That is, for every competent agent in the sequence, there is a lower-bounded probability of publishing. We refer to this as a \textit{uniformly nondegenerate gate}.

Let \(n:=T-1\). For agent \(i\), let \(K_i\) denote the number of correct private outcomes across the \(n\) learning trials.
\begin{definition}[Uniformly Nondegenerate Gate]
Fix $n=T-1$. 

Let $q_i(p) := \mathbb{P}(\text{Agent } i \text{ publishes} \mid \text{competence } p)$. The gate is \textit{uniformly nondegenerate} if there exists a constant $q_{\min} > 0$ such that for all agents $i$ and all $p \ge 0.5$:
\[
q_i(p) \ge q_{\min}.
\]
\end{definition}

\begin{theorem}
Consider a sequence of agents in the majority voting setting of Theorem \ref{maintheorem}. Fix the horizon $T \geq 2$. Assume that for the sequence of agents:
(1) The average competence satisfies $\frac{1}{N}\sum_{i=1}^N p_i \ge 0.5 + \Delta p$ for some $\Delta p > 0$; and
(2) The gate is uniformly nondegenerate.
Then, the probability of the electorate to identify the correct alternative in the final voting round $T$ converges to 1 as $N \to \infty$.
\end{theorem}

\textbf{Proof Sketch.} 
We proceed in \textit{four} steps:
(i) Under the distribution-level reliability model, the number of learning successes $K_i$ follows a Binomial distribution;
(ii) We show that the posterior confidence is strictly non-decreasing in $K_i$, implying the publish rule is a \textit{success-count threshold} ($K_i \ge k_i^\star$);
(iii) Using the uniform non-degeneracy and average competence assumptions, we lower-bound the expected vote margin;
(iv) Substituting this margin into the Azuma-Hoeffding bound yields an error term that decays exponentially with $N$.
The full proof is in: Appendix \ref{convergence}.

\begin{figure*}[t]
  \centering
  \begin{subfigure}[t]{0.48\textwidth}
    \centering
    \includegraphics[width=\linewidth]{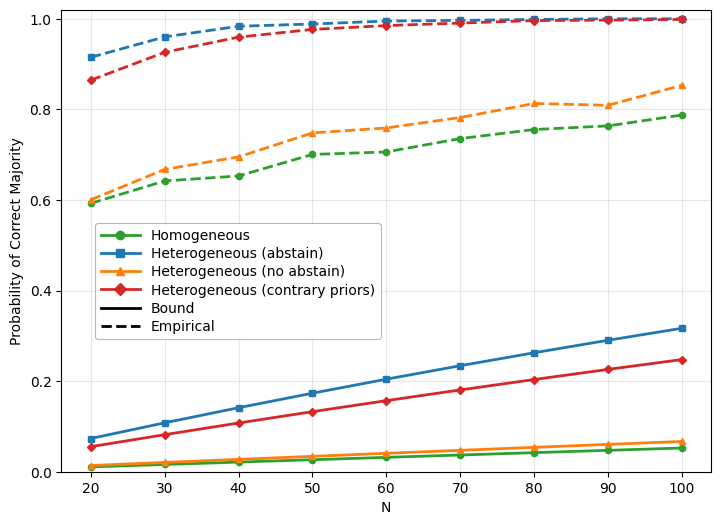}
    \caption{$T = 20, \bar{p} = 0.55, {{p_{\text{critical},i}} = 0.5},{\tau_{\text{abstain},i} = 0.5}$.}
    \label{fig:bound1a}
  \end{subfigure}
  \hfill
  \begin{subfigure}[t]{0.48\textwidth}
    \centering
    \includegraphics[width=\linewidth]{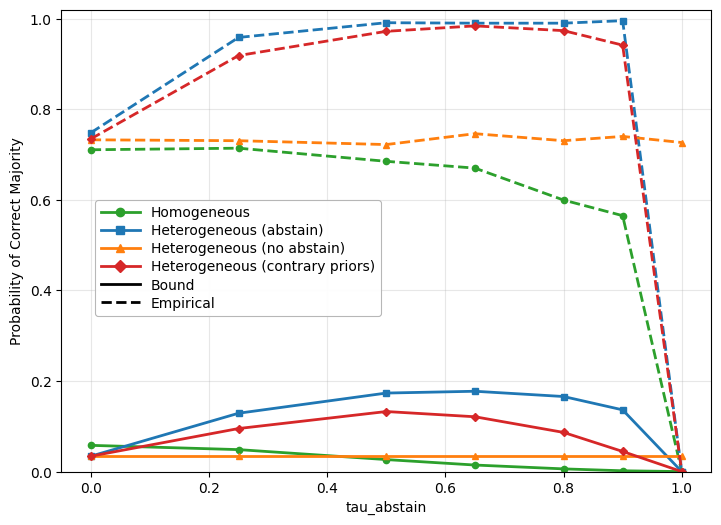}
    \caption{$T = 20, \bar{p} = 0.55, {{p_{\text{critical},i}} = 0.5}, N = 50$.}
    \label{fig:bound1b}
  \end{subfigure}
  \caption{Empirical Simulation and Equation \ref{finalbound} visualization under two configurations.}
  \label{fig:bound1}
\end{figure*}

\subsection{Empirical Simulations}

In this subsection, we empirically assess the bound from Theorem \ref{maintheorem} by Monte Carlo simulation of the full process, and compare this to the Azuma–Hoeffding lower bound (Equation \ref{finalbound}) computed from the same parameters. For a pseudocode description of the implementation, we refer the reader to subsection \ref{pseudocode} of the appendix. 
In each simulation round we: (i) draw all private outcomes for all agents over the $T\!-\!1$ learning tasks; (ii) update each agent’s Beta counts; (iii) apply the gate at time $T$ using $p_{\mathrm{critical},i}$ and $\tau_{\mathrm{abstain},i}$; (iv) sample each publishing agent’s final private vote on $\Omega_T$; and (v) record whether the final majority is correct. Repeating this $2000$ times yields the empirical success probability (the share of rounds in which the correct alternative wins).
We consider four agent pools, accounting for different competence distributions, and their interaction with the abstention gate, across two parameter configurations. 
That is, (i) Homogeneous agents (identical competencies); (ii) Heterogeneous agents with a prescribed ${\tau_{\text{abstain},i}}$-value, and uniform Beta priors ($(\alpha_{i,0},\beta_{i,0})=(1,1)$); (iii) Heterogeneous agents that never abstain ($\tau_{\mathrm{abstain}}=0$) as a baseline comparison; (iv) Heterogeneous agents with priors chosen contrary to their true reliability.
That is, setting (iv) studies miscalibrated Beta priors, implemented as follows: For agent $i$ with true reliability $p_i$, choose a prior strength $\kappa>0$ (total pseudo–count) and set (for $\epsilon = 10^{-6}$ to ensure $\alpha,\beta>0$ for numerical stability):
$
\alpha_{i,0}\ :=\ \max\!\big(\kappa(1-p_i),\,\varepsilon\big),
\beta_{i,0}\ :=\ \max\!\big(\kappa p_i,\,\varepsilon\big).
$ That is, competent agents ($p_i > 0.5$) start \textit{pessimistic}, and incompetent agents ($p_i < 0.5$) start \textit{optimistic}. For all settings, we assume an average reliability of $\bar{p} = 0.55$, $T = 20$ trials, and $p_{\mathrm{critical},i}=0.5$. Heterogeneity is implemented by assigning $p_i\in\{0.35,\,0.75\}$ to equal-sized halves, so that $\bar p=0.55$. 

In Figure \ref{fig:bound1a}, we plot the bound and empirical success probability as the number of agents ($N$) varies, prescribing $\tau_{\mathrm{abstain},i}=0.5$ for all agents. In Figure \ref{fig:bound1b}, we plot against a common $\tau_{\mathrm{abstain}}$ for $N = 50$ agents. 
Across both settings, the gated models (blue: heterogeneous with uniform priors; red: heterogeneous with contrary priors) dominate the non-abstaining baseline (orange) and the homogeneous pool (green): selective abstention suppresses low-competence voters, so the publishing electorate is, in expectation, stronger, leading to both a higher empirical success rate, and theoretical minimum success probability. An exception appears in Fig.~\ref{fig:bound1b} for very large gates: once $\tau_{\mathrm{abstain}}$ exceeds what even the better half ($p_i=0.75$) can typically certify, many agents stop publishing and performance degrades, while the non-abstaining baseline remains roughly flat.
In both plots, the empirical success rates (dashed) lie well above the theoretical curves (solid). This is expected as \Cref{finalbound} is a \textit{worst-case} concentration bound, protecting against the most unfavorable realizations by only using coarse information about the process.

\section{Summary and Future Work}

In this work, we introduced a probabilistic framework to analyze collective decision-making with endogenous participation. By employing martingales and filtrations, we captured the dynamic evolution of agents' confidence and their resulting decision to vote or abstain. Our primary contribution is the derivation of a non-asymptotic lower bound on the probability of a correct majority, alongside an asymptotic guarantee that extends the classical CJT to settings with heterogeneous agents and confidence-based abstention.

This framework offers a theoretical bridge between two distinct views on abstention: the \textit{strategic} view of social choice theory (where agents abstain to avoid pivoting the vote incorrectly) and the \textit{epistemic} view of statistical learning (where agents abstain due to low confidence). We showed that ``calibrated'' agents who simply maximize their own precision naturally improve the collective information state.

For future work, we plan to explore three avenues. First, from a theoretical standpoint, we aim to derive tighter concentration bounds by incorporating the variance structure of the martingale difference sequence (e.g., using \textit{Freedman's inequality}). Second, regarding social choice axioms, we wish to relax the independence assumption to model correlated information sources, a common scenario in committee deliberations. Finally, we plan to apply this framework to the design of hybrid intelligence systems, using Large Language Models (LLMs) as calibrated agents to empirically validate how confidence-gated voting mitigates group-level hallucination in collective decision-making. This will involve: (i) instantiating the agents as individual LLMs or as different reasoning paths of a single LLM, (ii) using a feedback loop to simulate the learning process and update their competence beliefs, and (iii) evaluating the success of our confidence-based abstention mechanism in mitigating hallucinations on a range of binary decision tasks.

\section{Acknowledgments}
This work is partly supported by BMFTR (Federal Ministry of Research, Technology and Space) in DAAD project 57616814 (SECAI, School of Embedded Composite AI, https://secai.org/) as part of the program Konrad Zuse Schools of Excellence in Artificial Intelligence.

\section*{Appendix} 

\subsection{Proof for Proposition 1.}\label{differencesequence}

We verify that $(D_k)_{k=1, \dots, NT}$ forms a martingale difference sequence with respect to the filtration $(\mathcal{H}_k)$. Recall that we let 

 $$Z = V^{\ost-\odg}_{T} - \mathbb{E}_\ost[V^{\ost-\odg}_{T}]$$
 
 be the centered random variable representing the total deviation of interest and that we can then construct a Doob martingale for this centered variable:

$$
M'_k = \mathbb{E}_\ost[Z \mid \mathcal{H}_k] = \mathbb{E}_\ost[V^{\ost-\odg}_{T} - \mathbb{E}_\ost[V^{\ost-\odg}_{T}] \mid \mathcal{H}_k]
$$

Each $D_k = M'_k - M'_{k-1}$ represents the change in the conditional expectation of the centered sum of final votes $V^{\ost-\odg}_{T} - \mathbb{E}_\ost[V^{\ost-\odg}_{T}]$ when the $k$-th private outcome is revealed. 

We show each property individually.

\medskip

1. \textit{Integrability ($\mathbb{E}_\ost[|D_k|] < \infty$)}:
   By the triangle inequality, $|D_k| = |M'_k - M'_{k-1}| \le |M'_k| + |M'_{k-1}|$.
   Recall that $M'_k = \mathbb{E}_\ost[V^{\ost-\odg}_{T} - \mathbb{E}_\ost[V^{\ost-\odg}_{T}] \mid \mathcal{H}_k]$.
   The random variable $V^{\ost-\odg}_{T}$ is a sum of $N$ terms (agents' votes at time $T$), each bounded within $\{-1, 0, 1\}$. Thus, $V^{\ost-\odg}_{T}$ is bounded within $[-N, N]$.
   Consequently, the centered variable $V^{\ost-\odg}_{T} - \mathbb{E}_\ost[V^{\ost-\odg}_{T}]$ is also bounded. Specifically, its values lie within $[-N - \mathbb{E}_\ost[V^{\ost-\odg}_{T}], N - \mathbb{E}_\ost[V^{\ost-\odg}_{T}]]$. Since $\mathbb{E}_\ost[V^{\ost-\odg}_{T}]$ is itself bounded within $[-N, N]$, the range of the centered variable is at most $[-2N, 2N]$.
   The conditional expectations $M'_k = \mathbb{E}_\ost[V^{\ost-\odg}_{T} - \mathbb{E}_\ost[V^{\ost-\odg}_{T}] \mid \mathcal{H}_k]$ and $M'_{k-1} = \mathbb{E}_\ost[V^{\ost-\odg}_{T} - \mathbb{E}_\ost[V^{\ost-\odg}_{T}] \mid \mathcal{H}_{k-1}]$ are therefore also bounded random variables. A bounded random variable has a finite expected absolute value. Consequently, $\mathbb{E}_\ost[|D_k|]$ is finite.

\medskip

2. \textit{Adaptedness ($D_k$ is measurable with respect to $\mathcal{H}_k$):}

A random variable $X$ is said to be $\mathcal{F}$-measurable if, for every Borel set $B$ in the measurable space $(\mathbb{R}, \mathcal{B}(\mathbb{R}))$, where $\mathcal{B}(\mathbb{R})$ refers to the collection of all subsets of $\mathbb{R}$ that satisfy the properties of a $\sigma$-algebra (containing the empty set and $\mathbb{R}$, and being closed under complementation and countable unions), and thus are precisely the sets to which a probability measure can be consistently assigned, the preimage $X^{-1}(B) = \{ s \in \mathsf S \mid X(s) \in B \}$ is an element of the $\sigma$-algebra $\mathcal{F}$ \citep{durrett2019probability}.
In essence, this means that for any possible value or range of values that $X$ can take, we can determine, solely by observing the events in $\mathcal{F}$, whether $X$ falls into that value or range.

In our specific context, the random variables $M'_k$ and $M'_{k-1}$ (and thus $D_k$) take on values that are conditional expectations of the centered variable $V^{\ost-\odg}_{T} - \mathbb{E}_\ost[V^{\ost-\odg}_{T}]$. A key property of conditional expectation is that $\mathbb{E}_\ost[Y \mid \mathcal{H}]$ is always $\mathcal{H}$-measurable for any random variable $Y$ for which the expectation exists \citep{durrett2019probability}.

Let's demonstrate why $D_k = M'_k - M'_{k-1}$ is $\mathcal{H}_k$-measurable. For $D_k$ to be $\mathcal{H}_k$-measurable, both $M'_k$ and $M'_{k-1}$ must be $\mathcal{H}_k$-measurable.

\begin{itemize}
 \item \textit{Measurability of $M'_k$ with respect to $\mathcal{H}_k$:}
 By the very definition of a Doob martingale, $M'_k = \mathbb{E}_\ost[V^{\ost-\odg}_{T} - \mathbb{E}_\ost[V^{\ost-\odg}_{T}] \mid \mathcal{H}_k]$ is, by construction, $\mathcal{H}_k$-measurable. Therefore, for any Borel set $B \in \mathcal{B}(\mathbb{R})$, the event $\{s \in \mathsf S \mid M'_k(s) \in B\}$ is an element of $\mathcal{H}_k$.

 \item \textit{Measurability of $M'_{k-1}$ with respect to $\mathcal{H}_k$:}
The random variable $M'_{k-1} = \mathbb{E}_\ost[V^{\ost-\odg}_{T} - \mathbb{E}_\ost[V^{\ost-\odg}_{T}] \mid \mathcal{H}_{k-1}]$ is, by definition, $\mathcal{H}_{k-1}$-measurable. Since $(\mathcal{H}_k)_{k=0, \dots, NT}$ is a filtration, it satisfies the property that $\mathcal{H}_{k-1} \subseteq \mathcal{H}_k$. This implies that any random variable that is measurable with respect to the smaller $\sigma$-algebra $\mathcal{H}_{k-1}$ is also measurable with respect to the larger $\sigma$-algebra $\mathcal{H}_k$. This is the case since for every Borel set $B$, it holds that if $X^{-1}(B) \in \mathcal{H}_{k-1}$, then $X^{-1}(B) \in \mathcal{H}_k$. Therefore, $M'_{k-1}$ is $\mathcal{H}_k$-measurable.
\end{itemize}

Since both $M'_k$ and $M'_{k-1}$ are $\mathcal{H}_k$-measurable random variables, their difference, $D_k = M'_k - M'_{k-1}$, is also $\mathcal{H}_k$-measurable. This follows from the fact that for any two measurable functions, their difference is also measurable \citep{salamon2016measure}.
This confirms that the sequence $(D_k)_{k=1, \dots, NT}$ is {adapted} to the filtration $(\mathcal{H}_k)_{k=0, \dots, NT}$.

\medskip

3. \textit{Zero Conditional Mean} ($\mathbb{E}_\ost[D_k \mid \mathcal{H}_{k-1}] = 0$):
 We expand $\mathbb{E}_\ost[D_k \mid \mathcal{H}_{k-1}]$:
\begin{align*}
\mathbb{E}_\ost[D_k \mid \mathcal{H}_{k-1}] &= \mathbb{E}_\ost[M'_k - M'_{k-1} \mid \mathcal{H}_{k-1}] \\
 &\expl{\text{(By linearity of conditional expectation)}} \\
 &= \mathbb{E}_\ost[M'_k \mid \mathcal{H}_{k-1}] - \mathbb{E}_\ost[M'_{k-1} \mid \mathcal{H}_{k-1}]
\end{align*}
Since $M'_{k-1}$ is measurable with respect to $\mathcal{H}_{k-1}$ (i.e., known given $\mathcal{H}_{k-1}$), $\mathbb{E}_\ost[M'_{k-1} \mid \mathcal{H}_{k-1}] = M'_{k-1}$.
For the first term, by the tower property of conditional expectation ($\mathbb{E}_\ost[\mathbb{E}_\ost[Y \mid \mathcal{H}_k] \mid \mathcal{H}_{k-1}] = \mathbb{E}_\ost[Y \mid \mathcal{H}_{k-1}]$ for any random variable $Y$ given that the filtration is ordered sequentially):
\begin{align*}
\mathbb{E}_\ost[M'_k \mid \mathcal{H}_{k-1}] &= \mathbb{E}_\ost[\mathbb{E}_\ost[V^{\ost-\odg}_{T} - \mathbb{E}_\ost[V^{\ost-\odg}_{T}] \mid \mathcal{H}_k] \mid \mathcal{H}_{k-1}] \\
    &= \mathbb{E}_\ost[V^{\ost-\odg}_{T} - \mathbb{E}_\ost[V^{\ost-\odg}_{T}] \mid \mathcal{H}_{k-1}] \\
    &= M'_{k-1}
\end{align*}
 Substituting these back, we find:
\[
\mathbb{E}_\ost[D_k \mid \mathcal{H}_{k-1}] = M'_{k-1} - M'_{k-1} = 0
\]
Thus, all three conditions are satisfied, confirming that $(D_k)$ is indeed a martingale difference sequence with respect to $(\mathcal{H}_k)$.

\subsection{Proof for Lemma 2.}\label{bounded}
Recall that we let 

 $$Z = V^{\ost-\odg}_{T} - \mathbb{E}_\ost[V^{\ost-\odg}_{T}]$$
 
 be the centered random variable representing the total deviation of interest and that we can then construct a Doob martingale for this centered variable:

$$
M'_k = \mathbb{E}_\ost[Z \mid \mathcal{H}_k] = \mathbb{E}_\ost[V^{\ost-\odg}_{T} - \mathbb{E}_\ost[V^{\ost-\odg}_{T}] \mid \mathcal{H}_k]
$$

Each $D_k = M'_k - M'_{k-1}$ represents the change in the conditional expectation of the centered sum of final votes $V^{\ost-\odg}_{T} - \mathbb{E}_\ost[V^{\ost-\odg}_{T}]$ when the $k$-th private outcome is revealed. 

As a general strategy to determine $c_k$ for $|D_k|$, we determine the maximal change in expectation we can obtain from a single martingale difference. More specifically, we consider a specific but arbitrarily chosen event $X_{(k)}$ in the event-based filtration and take into account the time of its occurrence.

Let $(i^*, t^*)$ denote the specific agent and time step corresponding to the $k$-th private outcome, $X_{(k)} = X_{i^*, t^*}$. The filtration $\mathcal{H}_k$ contains all information in $\mathcal{H}_{k-1}$ plus the outcome of $X_{(k)}$.

The variable $M'_k$ is the conditional expectation of $\sum_{j=1}^N (V_{j,T} - \mathbb{E}_\ost[V_{j,T}])$. When $X_{(k)} = X_{i^*, t^*}$ is revealed, it only affects the conditional expectation of $V_{i^*,T}$ (and consequently the term $V_{i^*,T} - \mathbb{E}_\ost[V_{i^*,T}]$). The conditional expectations of other agents' votes $V_{j,T}$ (for $j \neq i^*$) are independent of $X_{(k)}$ given $\mathcal{H}_{k-1}$ and thus their contribution to the change is zero.
Thus, the martingale difference $D_k$ simplifies to:
$$
D_k = \mathbb{E}_\ost[V_{i^*,T} - \mathbb{E}_\ost[V_{i^*,T}] \mid \mathcal{H}_k] - \mathbb{E}_\ost[V_{i^*,T} - \mathbb{E}_\ost[V_{i^*,T}] \mid \mathcal{H}_{k-1}]
$$
Since $\mathbb{E}_\ost[V_{i^*,T}]$ is a constant, i.e. an unconditional expectation (which we can denote as $\mu = \mathbb{E}_\ost[V_{i^*,T}]$), and recalling the linearity property of conditional expectation ($\mathbb{E}_\ost[A-B \mid \mathcal{F}] = \mathbb{E}_\ost[A \mid \mathcal{F}] - \mathbb{E}_\ost[B \mid \mathcal{F}]$) and that the conditional expectation of a constant is the constant itself ($\mathbb{E}_\ost[\mu \mid \mathcal{F}] = \mu$), we can expand the terms:
\begin{align*}
D_k &= \mathbb{E}_\ost[V_{i^*,T} - \mu \mid \mathcal{H}_k] - \mathbb{E}_\ost[V_{i^*,T} - \mu \mid \mathcal{H}_{k-1}] \\
    &= ( \mathbb{E}_\ost[V_{i^*,T} \mid \mathcal{H}_k] - \mathbb{E}_\ost[\mu \mid \mathcal{H}_k] ) 
    \\ & \quad - \left( \mathbb{E}_\ost[V_{i^*,T} \mid \mathcal{H}_{k-1}] - \mathbb{E}_\ost[\mu \mid \mathcal{H}_{k-1}] \right) \\
    &= \left( \mathbb{E}_\ost[V_{i^*,T} \mid \mathcal{H}_k] - \mu \right) - \left( \mathbb{E}_\ost[V_{i^*,T} \mid \mathcal{H}_{k-1}] - \mu \right) \\
    &= \mathbb{E}_\ost[V_{i^*,T} \mid \mathcal{H}_k] - \mu - \mathbb{E}_\ost[V_{i^*,T} \mid \mathcal{H}_{k-1}] + \mu \\
    &= \mathbb{E}_\ost[V_{i^*,T} \mid \mathcal{H}_k] - \mathbb{E}_\ost[V_{i^*,T} \mid \mathcal{H}_{k-1}]
\end{align*} 
The individual net public vote $V_{i^*,T}$ for agent $a_{i^*}$ at time $T$ is bounded within $\{-1, 0, 1\}$. Consequently, any conditional expectation of $V_{i^*,T}$ must also lie within the range $[-1, 1]$.

We now establish the bounds $c_k$ for $|D_k|$ based on the time step $t^*$ of the revealed outcome $X_{(k)}$:

\begin{enumerate}
    \item \textbf{If $X_{(k)}$ corresponds to agent $i^*$ at time $t^* < T$:}
    In this case, $X_{(k)} = X_{i^*, t^*}$ is a private choice made by agent $i^*$ at a time step prior to the final vote.
    We have $V_{i^*,T} = D_{\text{vote},i^*}(T) \cdot X_{i^*,T}$.
    The term $X_{i^*,T}$ represents agent $i^*$'s private choice at time $T$. 
    By our distribution-level reliability assumption, private choices are independent across trials. 
    This means $X_{i^*,T}$ is independent of all past private outcomes (including $X_{i^*,t^*}$ and the entire filtration $\mathcal{H}_k$) and $D_{\text{vote},i^*}(T)$ (which is determined by prior information), since the distribution of $X_{i^*,T}$ depends only on $\Omega_T$ (here $\Omega_T=\ost$) and not on past history.  Therefore, we can write: 

\begin{align*}
& \mathbb{E}_\ost[V_{i^*,T} \mid \mathcal{H}_k] \\
&= \mathbb{E}_\ost[D_{\text{vote},i^*}(T) \cdot X_{i^*,T} \mid \mathcal{H}_k] \\
&= \mathbb{E}_\ost[D_{\text{vote},i^*}(T) \cdot \mathbb{E}_\ost[X_{i^*,T} \mid D_{\text{vote},i^*}(T), \mathcal{H}_k] \mid \mathcal{H}_k] \\
&\quad \expl{Since $X_{i^*,T}$ is independent of all past outcomes} \\ & \quad \expl{which determine $D_{\text{vote},i^*}(T)$ and $\mathcal{H}_k$, the inner expectation simplifies to:} \\
&= \mathbb{E}_\ost[D_{\text{vote},i^*}(T) \cdot (2p_{i^*}-1) \mid \mathcal{H}_k] \\
&\quad \expl{Then, as $(2p_{i^*}-1)$ is a constant, it can be factored out of the outer expectation.} \\
&= (2p_{i^*}-1) \mathbb{E}_\ost[D_{\text{vote},i^*}(T) \mid \mathcal{H}_k]
\end{align*}

Here, we made use of the \textit{Tower Property of Conditional Expectation}. This property states that for any random variable $Y$ and nested $\sigma$-algebras $\mathcal{G}_1 \subseteq \mathcal{G}_2$:
$$ \mathbb{E}_\ost[Y \mid \mathcal{G}_1] = \mathbb{E}_\ost[\mathbb{E}_\ost[Y \mid \mathcal{G}_2] \mid \mathcal{G}_1] $$
In our derivation, $Y$ is the product $D_{\text{vote},i^*}(T) \cdot X_{i^*,T}$. We set $\mathcal{G}_1 = \mathcal{H}_k$ and choose $\mathcal{G}_2 = \sigma(D_{\text{vote},i^*}(T), \mathcal{H}_k)$. This specific choice of $\mathcal{G}_2$ is made to ensure that $D_{\text{vote},i^*}(T)$ is measurable with respect to the inner conditioning set, allowing it to be factored out from the inner expectation.

    The same logic applies to $\mathbb{E}_\ost[V_{i^*,T} \mid \mathcal{H}_{k-1}]$. Substituting these into the expression for $D_k$:
    
    \begin{align*}
    D_k &= (2p_{i^*}-1) \mathbb{E}_\ost[D_{\text{vote},i^*}(T) \mid \mathcal{H}_k] \\
    & \quad- (2p_{i^*}-1) \mathbb{E}_\ost[D_{\text{vote},i^*}(T) \mid \mathcal{H}_{k-1}] \\
    &= (2p_{i^*}-1) ( \mathbb{E}_\ost[D_{\text{vote},i^*}(T) \mid \mathcal{H}_k] \\ & \quad - \mathbb{E}_\ost[D_{\text{vote},i^*}(T) \mid \mathcal{H}_{k-1}] ).
    \end{align*}
    The terms $\mathbb{E}_\ost[D_{\text{vote},i^*}(T) \mid \mathcal{H}_k]$ and $\mathbb{E}_\ost[D_{\text{vote},i^*}(T) \mid \mathcal{H}_{k-1}]$ are conditional probabilities of $D_{\text{vote},i^*}(T)$ (a binary variable that is $0$ or $1$). Thus, they both lie in the interval $[0,1]$. The maximum possible absolute difference between any two values in this interval is $1 - 0 = 1$.
    Therefore, $|D_k| \leq |2p_{i}-1| \cdot 1 = |2p_{i}-1|$. We set $c_k = |2p_{i}-1|$. Note that $0 \le |2p_{i}-1| \le 1$.

    \item \textbf{If $X_{(k)}$ corresponds to agent $i^*$ at time $t^* = T$:}
    In this scenario, $X_{(k)}$ is $X_{i^*,T}$, the actual private choice of agent $i^*$ at the final time step $T$. When this outcome is revealed at step $k$, it directly determines the value of $V_{i^*,T}$ for agent $i^*$ (since $D_{\text{vote},i^*}(T)$ is already determined by information up to time $T-1$, making it $\mathcal{H}_{k-1}$-measurable).
    Specifically:
    \begin{itemize}
        \item For $\mathcal{H}_k$ we obtain: $\mathbb{E}_\ost[V_{i^*,T} \mid \mathcal{H}_k] = V_{i^*,T}$, because $X_{i^*,T}$ (and thus $V_{i^*,T}$) is now known given $\mathcal{H}_k$.
        \item For $\mathcal{H}_{k-1}$ we obtain: 
        \begin{align*}
        \mathbb{E}_\ost[V_{i^*,T} \mid \mathcal{H}_{k-1}] &= \mathbb{E}_\ost[D_{\text{vote},i^*}(T) \cdot X_{i^*,T} \mid \mathcal{H}_{k-1}] \\ &= D_{\text{vote},i^*}(T) \cdot \mathbb{E}_\ost[X_{i^*,T} \mid \mathcal{H}_{k-1}] \\ &= D_{\text{vote},i^*}(T) \cdot (2p_{i^*}-1). 
        \end{align*}
        This holds because $D_{\text{vote},i^*}(T)$ is $\mathcal{H}_{k-1}$-measurable, and $X_{i^*,T}$ is independent of $\mathcal{H}_{k-1}$ by the independence of task $\Omega_T$ from past tasks. 
    \end{itemize}
    Thus, $$D_k = V_{i^*,T} - D_{\text{vote},i^*}(T) \cdot (2p_{i^*}-1).$$
    
    Substituting $V_{i^*,T} = D_{\text{vote},i^*}(T) \cdot X_{i^*,T}$:
    
    \begin{align*}
    D_k &= D_{\text{vote},i^*}(T) \cdot X_{i^*,T} - D_{\text{vote},i^*}(T) \cdot (2p_{i^*}-1) \\ &= D_{\text{vote},i^*}(T) \cdot (X_{i^*,T} - (2p_{i^*}-1)).
    \end{align*}
    
    We consider the two possible values of $D_{\text{vote},i^*}(T)$:
    \begin{itemize}
        \item If $D_{\text{vote},i^*}(T) = 0$, then $D_k = 0$, so $|D_k| \le 2$ holds.
        \item If $D_{\text{vote},i^*}(T) = 1$, then $D_k = X_{i^*,T} - (2p_{i^*}-1)$.
        Since $X_{i^*,T} \in \{-1, 1\}$ and $p_{i^*} \in [0,1]$:
        \begin{itemize}
            \item If $X_{i^*,T} = 1$, $D_k = 1 - (2p_{i^*}-1) = 2 - 2p_{i} = 2(1-p_{i^*})$.
            \item If $X_{i^*,T} = -1$, $D_k = -1 - (2p_{i^*}-1) = -2p_{i^*}$.
        \end{itemize}
        The maximum absolute value of these terms is $\max(2(1-p_{i^*}), 2p_{i^*})$, which occurs at $p_{i^*}=0$ or $p_{i^*}=1$, and is equal to $2$.
        Thus, $|D_k| \le 2$.
    \end{itemize}
    We set $c_k = 2$.
\end{enumerate}

Now, we calculate the sum of squares of these bounds over all $NT$ events. We can group terms based on whether the private outcome $X_{(k)}$ occurred before time $T$ ($t^* < T$) or at time $T$ ($t^* = T$).
Let $K_{<T}$ be the set of indices $k$ where the corresponding time $t^* < T$, and $K_{=T}$ be the set of indices $k$ where the corresponding time $t^* = T$.

For $k \in K_{<T}$, we established that the bound $c_k = |2p_{i}-1|$. There are $(T-1)$ time steps (from $t=1$ to $t=T-1$) for each of the $N$ agents where this case applies. So, there are $N(T-1)$ such terms.
For $k \in K_{=T}$, we established that the bound $c_k = 2$. This case applies for each of the $N$ agents at the final time step $t=T$. So, there are $N$ such terms.

The sum of squares of these bounds over all $NT$ events is therefore:
$$\sum_{k=1}^{NT} c_k^2 = \sum_{k \in K_{<T}} c_k^2 + \sum_{k \in K_{=T}} c_k^2$$
Substituting the derived bounds for $c_k$:
$$\sum_{k=1}^{NT} c_k^2 = \sum_{i=1}^{N} \sum_{t=1}^{T-1} (2p_i-1)^2 + \sum_{i=1}^{N} 2^2$$
In the first sum, for each agent $i$, the term $(2p_i-1)^2$ appears $(T-1)$ times (once for each time step from $1$ to $T-1$). In the second sum, for each agent $i$, the term $2^2=4$ appears once (for the final time step $T$). This simplifies to:
$$\sum_{k=1}^{NT} c_k^2 = \sum_{i=1}^{N} \left( (T-1)(2p_i-1)^2 + 4 \right)$$

\subsection{Proof for the Hallucination Bound.}\label{halluproof}

In this proof, we utilize the upper-tail variant of the Azuma-Hoeffding inequality, restating Lemma \ref{Azuma} from above \citep{azuma1967weighted}:

Let $(Z_k)_{k=1}^K$ be a martingale difference sequence w.r.t.\ $(\mathcal H_k)$ with $|Z_k|\le c_k$ a.s.\ for all $k$. Then, for any $\epsilon>0$,
\begin{align*}
\mathbb P\!\left(\sum_{k=1}^{K} Z_k \ge \epsilon\right)\ \le\ 
\exp\!\left(-\frac{\epsilon^2}{2\sum_{k=1}^{K} c_k^2}\right).  
\end{align*}


For this derivation, we condition on $\Omega_T = \odg$. From Lemma \ref{Azuma} and Lemma \ref{lemmadifferences}, we obtain the following derivation:

\begin{align}
\mathbb{P}_\odg&\big({V}^{\ost-\odg}_{T} > 0\big) \nonumber \\
& \expl{{We subtract $\mathbb{E}_\odg[{V}^{\ost-\odg}_{T}]$ from both sides.}} \nonumber \\
& =  \mathbb{P}_\odg\!\left({V}^{\ost-\odg}_{T} - \mathbb{E}_\odg[{V}^{\ost-\odg}_{T}] \;\ge\; -\,\mathbb{E}_\odg[{V}^{\ost-\odg}_{T}]\right) \nonumber \\
& \expl{Apply the (upper-tail) Azuma--Hoeffding bound with $\epsilon = -\,\mathbb{E}_\odg[{V}^{\ost-\odg}_{T}]\ge 0$, $Z_k=D_k$,} \nonumber \\ & \expl{and $\sum c_k^2 = \sum_{i=1}^{N} \big( (T-1)(2p_i-1)^2 + 4 \big)$.} \nonumber \\
& \leq e^{\left(-\frac{\big(\mathbb{E}_\odg[{V}^{\ost-\odg}_{T}]\big)^2}{2 \sum_{i=1}^{N} \big( (T-1)(2p_i-1)^2 + 4 \big)}\right)} \nonumber \\
& \expl{Substitute the expected total net public vote under $\Omega_T=\odg$:} \nonumber \\
& \quad \expl{$\displaystyle \mathbb{E}_\odg[{V}^{\ost-\odg}_{T}] = \sum_{i=1}^{N} \mathbb{E}_\odg[V_{i,T}]
= - \sum_{i=1}^{N} (2p_i-1)\,\mathbb{E}_\odg[D_{\text{vote},i}(T)]$} \nonumber \\
& = e^{\left(-\frac{\left(\sum_{i=1}^{N} (2p_i-1)\,\mathbb{E}_\odg[D_{\text{vote},i}(T)]\right)^2}{2 \sum_{i=1}^{N} \big( (T-1)(2p_i-1)^2 + 4 \big)}\right)}. \label{eq:hallucination-bound}
\end{align}

\subsection{Proof for Theorem 3.}\label{convergence}

As before, we condition on the final task's true state $\Omega_T=\ost$ and write $\mathbb P_\ost(\cdot):=\mathbb P(\cdot\mid \Omega_T=\ost)$ and $\mathbb E_\ost[\cdot]$ accordingly. Recall the one-sided Azuma--Hoeffding lower bound derived in Equation \ref{finalbound}. We seek to show that as $N \to \infty$, the probability of a correct majority converges to 1.

We recall the underlying parameters: $V_T^{\ost-\odg}=\sum_{i=1}^N V_{i,T}$ is the final net vote, $p_i\in[0,1]$ is agent $i$'s distribution-level reliability, and $D_{\mathrm{vote},i}(T)\in\{0,1\}$ is the publication decision. The term $(2p_i-1)$ represents the agent's expected margin.

Let $n:=T-1$ be the number of learning rounds. The agent's decision to publish at time $T$ depends on the number of correct predictions $K_i$ observed during learning.

To prove convergence, we proceed in \textit{four} steps:

(i) establish that the number of successes $K_i$ follows a Binomial distribution;
(ii) show that the confidence $\mathcal{C}_{i,T}$ is non-decreasing in $K_i$, implying the publish rule is a \textit{success-count threshold};
(iii) bound the expected publish probability and vote margin using the uniform non-degeneracy assumption;
(iv) apply these bounds to the Azuma-Hoeffding result to show exponential convergence.

\smallskip
\textit{Step 1: Distribution of Successes.}
Under the \textit{distribution-level reliability} definition, the outcomes of the $n$ learning tasks are i.i.d.\ Bernoulli($p_i$) trials. Thus, for any agent $i$, the success count $K_i$ follows a binomial distribution:
\[
K_i \sim \mathrm{Binomial}(n,p_i).
\]

\smallskip
\textit{Step 2: Monotonicity and Threshold Rule.}
After observing $K_i=k$ successes, the posterior is $\Psi_{i,T} \sim \mathrm{Beta}(\alpha_{i,0}+k,\ \beta_{i,0}+n-k)$. The agent publishes if its confidence $g_i(k) := \mathbb P(\Psi_{i,T} > p_{\mathrm{critical},i})$ exceeds $\tau_{\mathrm{abstain},i}$.
For a fixed critical threshold $x$, the regularized incomplete beta function $I_x(\alpha, \beta)$ is decreasing in $\alpha$ and increasing in $\beta$. Since observing more successes ($k \uparrow$) increases $\alpha$ and decreases $\beta$, the tail probability $g_i(k) = 1 - I_{p_{\mathrm{critical},i}}(\dots)$ is strictly non-decreasing in $k$.

Consequently, the condition $g_i(k) > \tau_{\mathrm{abstain},i}$ defines a \textit{success-count threshold}: there exists some $k_i^\star$ such that the agent publishes if and only if $K_i \ge k_i^\star$.

\smallskip
\textit{Step 3: Bounding the Publish Probability.}

We define the publish probability for agent $i$ as a function of their competence $p$:
\[
q_i(p) := \mathbb{P}(K_i \ge k_i^\star \mid p) = \sum_{k=k_i^\star}^{n}\binom{n}{k}p^k(1-p)^{n-k}.
\]

\begin{lemma}[Convergence]
\label{lem:exp-conv-ug}
Fix the horizon $T$ and the gate parameters ($\alpha_0, \beta_0, p_{\mathrm{crit}}, \tau$). Consider a sequence of agents $1, \dots, N$ satisfying two conditions:
\begin{enumerate}
    \item \textbf{Average Competence:} The average reliability exceeds $1/2$ by a margin $\Delta p > 0$:
    \[ \frac{1}{N}\sum_{i=1}^N p_i \ge \frac{1}{2} + \Delta p. \]
    \item \textbf{Uniform Non-degeneracy:} The gate is \textit{uniformly nondegenerate}, meaning there exists a lower bound $q_{\min} > 0$ such that for all agents $i$ and all $p \ge 1/2$, the probability of publishing is at least $q_{\min}$ (i.e., $\inf_{i} q_i(p_i) \ge q_{\min}$ for competent agents).
\end{enumerate}
Then, as $N \to \infty$,
\begin{equation}
\label{eq:exp-conv-ug}
\mathbb P_\ost\!\big(V_T^{\ost-\odg}>0\big)
\ \ge\
1-\exp\!\left(-\frac{2\,\Delta p^2\,q_{\min}^{\,2}}{T+3}\,N\right)
\ \xrightarrow{}\ 1 .
\end{equation}
\end{lemma}

\begin{proof}
Recall the exponent from the Azuma-Hoeffding lower bound (Eq.~\ref{finalbound}). We analyze the numerator and denominator separately.

\textbf{Denominator:} The variance term for each agent is bounded by $(T-1)(2p_i-1)^2 + 4 \le (T-1) + 4 = T+3$. Summing over $N$ agents:
\[
\sum_{i=1}^{N} c_i^2 \le N(T+3).
\]

\textbf{Numerator:} The expected collective vote is $\sum_{i=1}^N (2p_i-1)\mathbb{E}_\ost[D_{\mathrm{vote},i}(T)]$.
Note that $\mathbb{E}_\ost[D_{\mathrm{vote},i}(T)] = q_i(p_i)$.
We split the sum into competent ($p_i \ge 1/2$) and incompetent ($p_i < 1/2$) agents.
For incompetent agents, $(2p_i-1) < 0$, so their contribution to the sum is negative (or zero).
For competent agents, we use the uniform non-degeneracy assumption: $q_i(p_i) \ge q_{\min}$.
Thus:
\begin{align*}
\sum_{i=1}^N (2p_i-1)q_i(p_i) 
&\ge \sum_{i: p_i \ge 1/2} (2p_i-1)q_i(p_i) + \sum_{i: p_i < 1/2} (2p_i-1)q_i(p_i) \\
&\ge q_{\min} \sum_{i: p_i \ge 1/2} (2p_i-1) + \sum_{i: p_i < 1/2} (2p_i-1) \cdot 1 \\
&\expl{(Since $q_i \le 1$ and terms are negative, replacing $q_i$ with 1 lowers the sum)} \\
&= q_{\min} \left( \sum_{i: p_i \ge 1/2} (2p_i-1) + \sum_{i: p_i < 1/2} (2p_i-1) \right) \\
&\expl{(Assuming $q_{\min} \le 1$, we can factor it out safely)} \\
&= 2 q_{\min} \sum_{i=1}^N (p_i - 1/2).
\end{align*}
Using the average competence assumption $\sum (p_i - 1/2) = N(\bar{p} - 1/2) \ge N \Delta p$, we get:
\[
\text{Numerator} \ge 2 q_{\min} N \Delta p.
\]

\smallskip
\textit{Step 4: Convergence.}

Substituting these into the Azuma-Hoeffding bound:
\[
\mathbb P_\ost > 1 - \exp\left( - \frac{(2 q_{\min} N \Delta p)^2}{2 \cdot N(T+3)} \right) = 1 - \exp\left( - \frac{2 \Delta p^2 q_{\min}^2}{T+3} N \right).
\]
Since $\Delta p, q_{\min} > 0$ and $T$ is finite, the exponent grows linearly with $N$, driving the error probability to 0.
\end{proof}

\subsection{Pseudocode for Empirical Evaluation.}\label{pseudocode}

We evaluate the framework using three routines: a per–agent
gate-and-contribution calculator (Alg.~1), a lower bound aggregator (Alg.~2),
and a Monte–Carlo simulator of the full sequential process (Alg.~3). 

We vary a single control parameter (e.g.\ $N$, $T$, $\tau_{\mathrm{abstain}}$) while
keeping the others fixed, and report \textit{four} scenarios: (i) homogeneous agents ($p_i\equiv 0.55$);
(ii) heterogeneous with abstention (half $p=0.35$, half $p=0.75$); (iii) the same heterogeneous
pool with abstention disabled; and (iv) heterogeneous with \textit{contrary (competence-inverse) priors}
used during learning. For each setting we plot the probabilistic lower bound from
Algs.~1--2 alongside the Monte–Carlo estimate from Alg.~3. Consistently, the gate filters
low-competence agents, increasing the expected margin and tightening the lower bound; both
empirical success and the certificate improve with larger $N$ and $T$.

\paragraph{Contrary priors (used in scenario (iv)).}
For agent $i$ with true reliability $p_i$, we set the Beta prior mean to $1-p_i$ with strength $\kappa>0$, and we clip Beta pseudo–counts below by $\epsilon = 10^{-6}$ to ensure $\alpha,\beta>0  $ for numerical stability:
\[
\alpha_{i,0}\ :=\ \max\!\big(\kappa(1-p_i),\,\varepsilon\big),\qquad
\beta_{i,0}\ :=\ \max\!\big(\kappa p_i,\,\varepsilon\big),
\]
so competent agents start pessimistic and weak agents optimistic. (When not using contrary priors
we employ aligned, shared priors $(\alpha_0,\beta_0)$ for all agents.) 
For instance, with $\kappa=8$ and a competent agent $p_i=0.75$, we set
\[
\alpha_{i,0}= \kappa(1-p_i)=8\cdot 0.25=2,\qquad
\beta_{i,0}= \kappa p_i=8\cdot 0.75=6,
\]
so the prior $\mathrm{Beta}(2,6)$ is pessimistic despite $p_i>1/2$.
Conversely, for a weak agent $p_i=0.35$,
\[
\alpha_{i,0}= 8\cdot(1-0.35)=5.2,\qquad
\beta_{i,0}= 8\cdot 0.35=2.8,
\]
yielding an optimistic prior $\mathrm{Beta}(5.2,2.8)$. 


\textbf{Alg.~1 (AgentGateAndContribution).}
Fix horizon $T$ (with $n=T-1$ learning rounds), agent $i$'s prior
$(\alpha_{i,0},\beta_{i,0})$ (aligned or contrary), competence $p_i$, target level $p_{\mathrm{critical},i}$,
and abstention threshold $\tau_{\mathrm{abstain},i}$.
The routine enumerates all learning histories with $k\in\{0,\dots,n\}$ correct private
outcomes. For each $k$ it forms the posterior $\mathrm{Beta}(\alpha_{i,0}+k,\beta_{i,0}+n-k)$
and computes the agent's \textit{confidence}
$\mathcal C_{i,T}=1-I_{p_{\mathrm{critical},i}}(\alpha_{i,0}+k,\beta_{i,0}+n-k)$.
If $\mathcal C_{i,T}>\tau_{\mathrm{abstain},i}$ (or if publishing is forced), that
history counts as ``publish''. Weighting by the binomial likelihood
$\binom{n}{k}p_i^k(1-p_i)^{n-k}$ and summing over $k$ yields
$q_i=\mathbb E_\ost[D_{\mathrm{vote},i}(T)]$, the probability that agent $i$ will
\textit{publish} at time $T$ under the model. Multiplying by the margin $(2p_i-1)$ gives
the agent's expected contribution $c_i=(2p_i-1)q_i$ to the final net vote. The update
\(
q_i \leftarrow q_i + \mathbf{1}\{\text{publish at }k\}\binom{n}{k}p_i^k(1-p_i)^{n-k}
\)
implements the law of total probability:
\begin{align*}
&q_i=\mathbb{E}_\ast[D_{\mathrm{vote},i}(T)]
= \\ & \sum_{k=0}^{n}\underbrace{\mathbf{1}\{\mathcal C_{i,T}(k)>\tau_{\mathrm{abstain},i}\}}_{\mathbb{P}(\text{publish}\mid K_i=k)}
\underbrace{\binom{n}{k}p_i^k(1-p_i)^{n-k}}_{\mathbb{P}(K_i=k)},
\end{align*}
that is, we add the probability of each history that would vote publicly.

\begin{algorithm}[!htbp]
\SetKwInOut{Input}{In}\SetKwInOut{Output}{Out}
\caption{AgentGateAndContribution$(i)$}
\Input{$T,\,\alpha_{i,0},\beta_{i,0},\,p_i,\,p_{\mathrm{critical},i},\,\tau_{\mathrm{abstain},i}$; \texttt{force}$\in\{0,1\}$}
\Output{$q_i=\mathbb{E}_\ast[D_{\text{vote},i}(T)]$, \ $c_i=(2p_i-1)\,q_i$}
$n\leftarrow T-1$;\ $q_i\leftarrow 0$\;
\For{$k=0,\dots,n$}{
  $\alpha\_\mathrm{post}\leftarrow \alpha_{i,0}+k$;\quad $\beta\_\mathrm{post}\leftarrow \beta_{i,0}+n-k$\;
  $\text{conf}\leftarrow 1-\mathrm{BetaCDF}\!\big(p_{\mathrm{critical},i};\,\alpha\_\mathrm{post},\beta\_\mathrm{post}\big)$\;
  $\text{publish}\leftarrow (\texttt{force}=1)\ \lor\ (\text{conf}>\tau_{\mathrm{abstain},i})$\;
  $q_i\leftarrow q_i+\mathbf{1}\{\text{publish}\}\,\binom{n}{k}p_i^k(1-p_i)^{n-k}$\;
}
$c_i\leftarrow (2p_i-1)\,q_i$\;
\Return{$(q_i,c_i)$}
\end{algorithm}

\textbf{Alg.~2 (BoundFromContrib).}
Given the per–agent contributions $\{c_i\}$ and competences $\{p_i\}$, the routine
computes the numerator $S_{\mathrm{num}}=\sum_i c_i=\mathbb E_\ost[V_T^{\ost-\odg}]$
and the Azuma–Hoeffding denominator
$S_{\mathrm{den}}=\sum_i\big((T-1)(2p_i-1)^2+4\big)$, which collects the bounded
differences from the $n$ learning reveals and the final private vote (Lemma~\ref{lemmadifferences}).
It returns the lower bound
$1-\exp\!\big(-S_{\mathrm{num}}^2/(2S_{\mathrm{den}})\big)$ on the probability
that the final majority is correct. 

\begin{algorithm}[!htbp]
\SetKwInOut{Input}{In}\SetKwInOut{Output}{Out}
\caption{BoundFromContrib}
\Input{$T$;\ $\{p_i\}_{i=1}^N$;\ $\{c_i\}_{i=1}^N$ with $c_i=(2p_i-1)q_i$}
\Output{$\mathrm{LB}\ \approx\ \text{lower bound on }\mathbb{P}_\ast(V_T^{\ost-\odg}>0)$}
$S_{\mathrm{num}}\leftarrow \sum_{i=1}^N c_i$\;
$S_{\mathrm{den}}\leftarrow \sum_{i=1}^N \big((T-1)(2p_i-1)^2+4\big)$\;
$\mathrm{LB}\leftarrow 1-\exp\!\big(-S_{\mathrm{num}}^2/(2S_{\mathrm{den}})\big)$\;
\Return{$\mathrm{LB}$}
\end{algorithm}


\textbf{Alg.~3 (MonteCarloEvaluate).}
This routine simulates the full process end-to-end. In each run it initializes all
agents with either a shared prior $(\alpha_0,\beta_0)$ \textit{or} per–agent priors
$\{(\alpha_{i,0},\beta_{i,0})\}$ (e.g., contrary priors as above), then executes $n$ learning rounds.
In round $t\le n$, agent $i$ draws a private outcome (correct with probability $p_i$) and updates its Beta
counts accordingly. At the decision round $T$, each agent computes its posterior
confidence $1-I_{p_{\mathrm{critical}}}(\alpha_i,\beta_i)$ and publishes iff it
exceeds $\tau_{\mathrm{abstain}}$ (or publishing is forced). Publishing agents then cast
a final private vote ($+1$ w.p.\ $p_i$, $-1$ otherwise). The net vote is the sum of
published votes; a run is a \textit{win} if the net number of votes is strictly positive. Repeating over
$R$ runs yields the empirical success rate, referred to as $\widehat{\mathbb{P}}_\mathrm{emp}$.


\begin{algorithm}[!htbp]
\SetKwInOut{Input}{In}\SetKwInOut{Output}{Out}
\caption{MonteCarloEvaluate}
\Input{$N,T,\,p_{\mathrm{critical}},\tau_{\mathrm{abstain}}$; list $\{p_i\}_{i=1}^N$; \texttt{force}$\in\{0,1\}$; runs $R$; \\
\textit{Either} shared prior $(\alpha_0,\beta_0)$ \textit{or} per–agent priors $\{(\alpha_{i,0},\beta_{i,0})\}$}
\Output{$\widehat{\mathbb{P}}_\mathrm{emp}(V_T^{\ost-\odg}>0)$}
$\text{wins}\leftarrow 0$;\quad $n\leftarrow T-1$\;
\For{$r=1,\dots,R$}{
  \tcc{learning phase}
  \For{$i=1,\dots,N$}{
    $(\alpha_i,\beta_i)\leftarrow
    \begin{cases}
      (\alpha_{i,0},\beta_{i,0}) & \text{if per–agent priors given}\\
      (\alpha_0,\beta_0) & \text{otherwise}
    \end{cases}$\;
  }
  \For{$t=1,\dots,n$}{
    \For{$i=1,\dots,N$}{
      $s\sim\mathrm{Bernoulli}(p_i)$;\quad
      $(\alpha_i,\beta_i)\leftarrow(\alpha_i+1,\beta_i)$ if $s{=}1$ else $(\alpha_i,\beta_i+1)$\;
    }
  }
  \tcc{decision phase}
  $\text{net}\leftarrow 0$\;
  \For{$i=1,\dots,N$}{
    $\text{conf}\leftarrow 1-\mathrm{BetaCDF}\!\big(p_{\mathrm{critical}};\alpha_i,\beta_i\big)$\;
    $\text{publish}\leftarrow (\texttt{force}=1)\ \lor\ (\text{conf}>\tau_{\mathrm{abstain}})$\;
    \If{$\text{publish}$}{
      $x_T \leftarrow +1 \quad\text{with probability } p_i \text{ else } -1$; \: $\text{net}\leftarrow \text{net}+x_T$\;
    }
  }
  $\text{wins}\leftarrow \text{wins} + \mathbf{1}\{\text{net}>0\}$\;
}
\Return{$\text{wins}/R$}
\end{algorithm}

\clearpage

\bibliographystyle{plainnat}
\bibliography{references}

\end{document}